\documentclass{article}

\usepackage{nips14submit_e,times}
\usepackage{caption}
\usepackage[utf8]{inputenc} 
\usepackage[T1]{fontenc}    
\usepackage{url}            
\usepackage{booktabs}       
\usepackage{amsfonts}       
\usepackage{nicefrac}       
\usepackage{microtype}      
\usepackage{times}
\usepackage{tabularx,tabulary}
\usepackage{graphicx}
\usepackage{subfigure} 
\usepackage{amsthm}
\usepackage{amsmath}
\usepackage{amssymb}
\usepackage{enumitem}
\usepackage{algorithm}
\usepackage{algorithmic}

\usepackage{hyperref}
\usepackage{url}
\usepackage{cite}

\title{A Minimax Approach to Supervised Learning}

\author{
Farzan Farnia\thanks{Department of Electrical Engineering, Stanford University, Stanford, CA  94305.} \\
\texttt{farnia@stanford.edu} \\
\And
David Tse$^*$ \\
\texttt{dntse@stanford.edu} \\
}

\makeatletter
\newenvironment{subtheorem}[1]{%
  \def\subtheoremcounter{#1}%
  \refstepcounter{#1}%
  \protected@edef\theparentnumber{\csname the#1\endcsname}%
  \setcounter{parentnumber}{\value{#1}}%
  \setcounter{#1}{0}%
  \expandafter\def\csname the#1\endcsname{\theparentnumber.\Alph{#1}}%
  \ignorespaces
}{%
  \setcounter{\subtheoremcounter}{\value{parentnumber}}%
  \ignorespacesafterend
}
\makeatother
\newcounter{parentnumber}

\newtheorem{thm}{Theorem}
\newtheorem{cor}{Corollary}

\newcommand{\E}{\mathbb{E}}
\newcommand{\X}{\mathcal{X}}

\newcommand{\A}{\mathcal{A}}

\nipsfinalcopy 

\begin{document}

\maketitle

\begin{abstract}
Given a task of predicting $Y$ from $X$, a  loss function $L$, and a set of probability distributions $\Gamma$ on $(X,Y)$, what is the optimal decision rule minimizing the worst-case expected loss over $\Gamma$? In this paper, we address this question by introducing a generalization of the principle of maximum entropy. Applying this principle to sets of distributions with marginal on $X$ constrained to be the empirical marginal from the data, we develop a general minimax approach for supervised learning problems. While for some loss functions such as squared-error and log loss, the minimax approach rederives well-knwon regression models, for the 0-1 loss it results in a new linear classifier which we call the maximum entropy machine. The maximum entropy machine minimizes the worst-case 0-1 loss over the structured set of distribution, and by our numerical experiments can outperform other well-known linear classifiers such as SVM. We also prove a bound on the generalization worst-case error in this minimax framework.
\end{abstract}

\section{Introduction}

Supervised learning, the task of inferring a function that predicts a target $Y$ from a feature vector $\mathbf{X}=(X_1,\ldots,X_d)$ by using $n$ labeled training samples $\lbrace(\mathbf{x}_1,y_1),\ldots ,(\mathbf{x}_n,y_n)\rbrace$, has been a problem of central interest in machine learning. Given the underlying distribution $\tilde{P}_{\mathbf{X},Y}$, the optimal prediction rules had long been studied and formulated in the statistics literature. However, the advent of high-dimensional problems raised this important question: What would be a good prediction rule when we do not have enough samples to estimate the underlying distribution?

To understand the difficulty of learning in high-dimensional settings, consider a genome-based classification task where we seek to predict a binary trait of interest $Y$ from an observation of $3,000,000$ SNPs, each of which can be considered as a discrete variable $X_i \in \lbrace 0,1,2 \rbrace$. Hence, to estimate the underlying distribution we need $O(3^{3,000,000})$ samples.

With no possibility of estimating the underlying $P^*$ in such problems, several approaches have been proposed to deal with high-dimensional settings. The standard approach in statistical learning theory is empirical risk minimization (ERM) \cite{vapnik}. ERM learns the prediction rule by minimizing an approximated loss under the empirical distribution of samples. However, to avoid overfitting, ERM restricts the set of allowable decision rules to a class of functions with limited complexity measured through its VC-dimension. However, the ERM problem for several interesting loss functions such as 0-1 loss is computationally intractable \cite{NP_01}.

This paper focuses on a complementary approach to ERM where one can learn the prediction rule through minimizing a decision rule's worst-case loss over a larger set of distributions $\Gamma(\hat{P})$ centered at the empirical distribution $\hat{P}$. 
In other words, instead of restricting the class of decision rules,  we consider and evaluate all possible decision rules, but based on a more stringent criterion that they will have to perform well over all distributions in $\Gamma(\hat{P})$. As seen in Figure \ref{Fig: MCE}, this minimax approach can be broken into three main steps: 
\begin{enumerate} [noitemsep,topsep=0pt]
\item We compute the empirical distribution $\hat{P}$ from the data,
\item We form a distribution set $\Gamma(\hat{P})$ based on $\hat{P}$,
\item We learn a prediction rule $\psi^*$ that minimizes the worst-case expected loss over $\Gamma(\hat{P})$.
\end{enumerate}

An important example of the above minimax approach is the linear regression  fitted via the least-squares method, which also minimizes the worst-case squared-error over all distributions with the same first and second order moments as empirically estimated from the samples. Some other special cases of this minimax approach, which are also based on learning a prediction rule from low-order marginal/moments, have been addressed in the literature: 
\cite{mpm} solves a robust minimax classification problem for continuous settings with fixed first and second-order moments; \cite{DCC} develops a classification approach by minimizing the worst-case hinge loss subject to fixed low-order marginals; \cite{DRC} fits a model minimizing the maximal correlation under fixed pairwise marginals to design a robust classification scheme. In this paper, we develop a general minimax approach for supervised learning problems with arbitrary loss functions.

To formulate Step 3 in Figure \ref{Fig: MCE}, given a general loss function $L$ and set of distribution $\Gamma(\hat{P})$ we generalize the problem formulation discussed at \cite{DCC} to
\begin{equation} \label{general minimax}
\underset{\psi \in \boldsymbol{\Psi}}{\arg\!\min} \: \max_{P\in \Gamma(\hat{P})}\: \E\left[\, L\bigl( Y,\psi(\mathbf{X})\bigr)\, \right].
\end{equation} 
Here, $\boldsymbol{\Psi}$ is the space of all decision rules. Notice the difference with the ERM problem where $\boldsymbol{\Psi}$ was restricted to smaller function classes while $\Gamma(\hat{P})=\lbrace \hat{P} \rbrace$.

If we have to predict $Y$ with no access to $\mathbf{X}$, \eqref{general minimax} reduces to
\begin{equation} \label{general minimax, unconditional}
\underset{a \in \mathcal{A}}{\min} \: \max_{P\in \Gamma(\hat{P})}\: \E\left[\, L\bigl( Y,a)\bigr)\, \right],
\end{equation}
where $\mathcal{A}$ is the action space for loss function $L$. For $L$ being the logarithmic loss function (log loss), Topsoe \cite{topsoe1979information} reduces \eqref{general minimax, unconditional} to the entropy maximization problem over $\Gamma(\hat{P})$. This result is shown based on Sion's minimax theorem \cite{sion} which shows under some mild conditions one can exchange the order of min and max in the minimax problem. Note that when $L$ is log loss, the maximin problem corresponding to \eqref{general minimax, unconditional} results in a maximum entropy problem. 
 More generally, this result provides a game theoretic interpretation of the principle of maximum entropy introduced by Jaynes in \cite{jaynes}. By the principle of maximum entropy, one should select and act based on a distribution in $\Gamma(\hat{P})$ which maximizes the Shannon entropy.

Grünwald and Dawid \cite{MaxEnt} generalize the minimax theorem for log loss in \cite{topsoe1979information} to other loss functions, showing \eqref{general minimax, unconditional} and its corresponding maximin problem have the same solution for a large class of loss functions. They further interpret the maximin problem as maximizing a generalized entropy function, which motivates generalizing the principle of maximum entropy for other loss functions: Given loss function $L$, select and act based on a distribution maximizing the generalized entropy function for $L$. Based on their minimax interpretation, the maximum entropy principle can be used for a general loss function to find and interpret the optimal action minimizing the worst-case expected loss in \eqref{general minimax, unconditional}. 

\begin{figure}
  \centering
  \begin{minipage}[t]{0.4\textwidth}
   \includegraphics[width=1.4\textwidth]{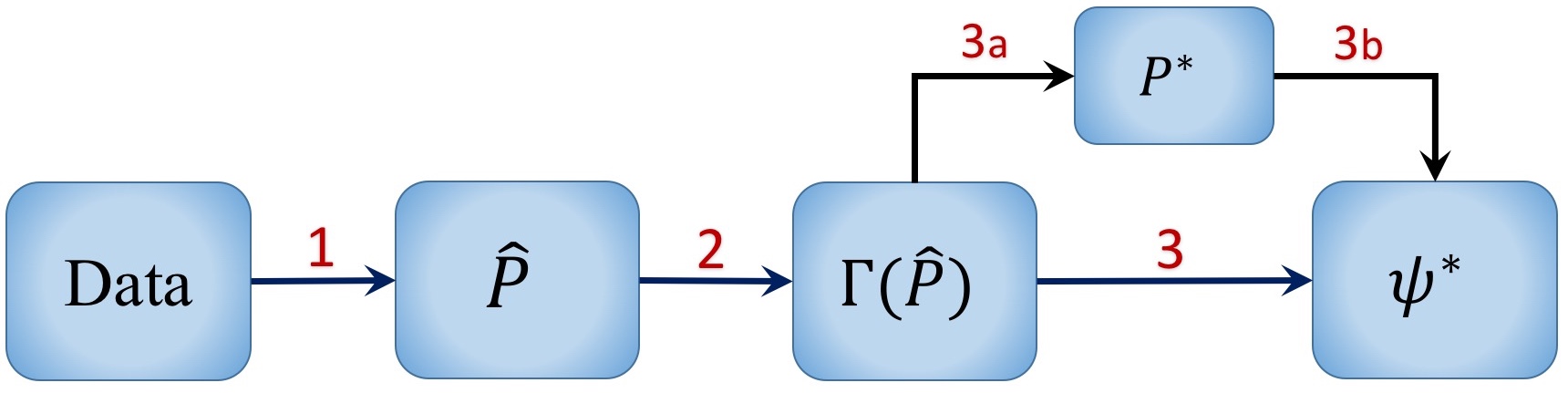}
     \captionsetup{justification=centering}
 \caption{Minimax Approach}    \label{Fig: MCE}
  \end{minipage}
  \hfill
  \begin{minipage}[t]{0.4\textwidth}
    \includegraphics[width=1\textwidth]{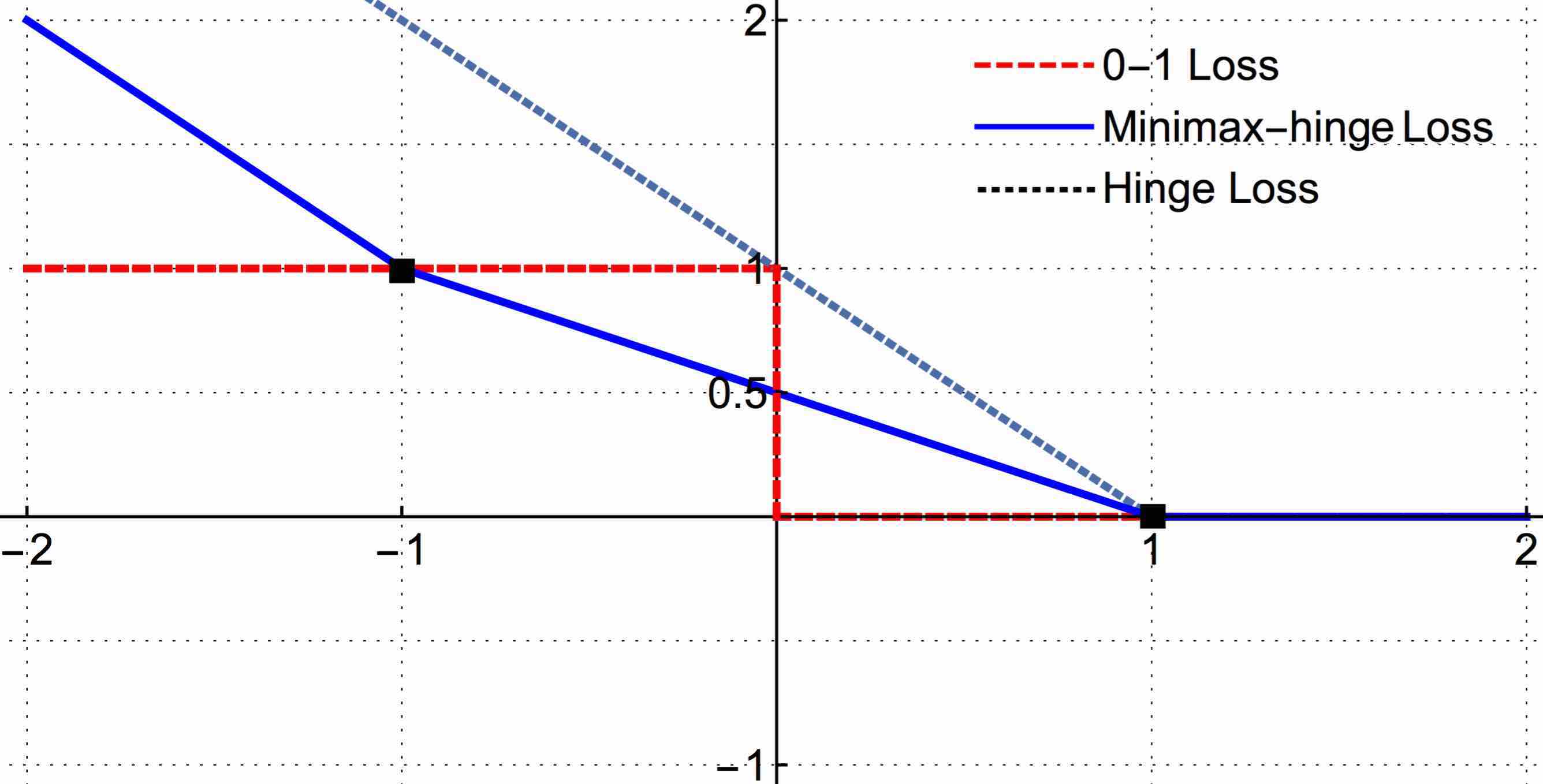}
      \captionsetup{justification=centering}
    \caption{Minimax-hinge Loss} \label{Fig:DHL}
  \end{minipage}
\end{figure}

How can we use the principle of maximum entropy to solve \eqref{general minimax}  where we observe $\mathbf{X}$ as well? 
A natural idea is to apply the maximum entropy principle to the conditional $P_{Y\vert \mathbf{X} =\mathbf{x}}$ instead of the marginal $P_Y$. This idea motivates a generalized version of the principle of maximum entropy, which we call the \emph{principle of maximum conditional entropy}. The conditional entropy maximization for prediction problems was first introduced and interpreted by Berger et al. \cite{berger}. They indeed proved that the logistic regression model maximizes the conditional entropy over a particular set of distributions. In this work, we extend the minimax interpretation from the maximum entropy principle \cite{topsoe1979information,MaxEnt} to the maximum conditional entropy principle, which reveals how the maximum conditional entropy principle breaks Step 3 into two smaller steps: 
\begin{itemize} [noitemsep,topsep=0pt]
\item[3a. ] We search for $P^*$ the distribution maximizing the conditional entropy over $\Gamma(\hat{P})$,
\item[3b. ] We find $\psi^*$ the optimal decision rule for $P^*$.
\end{itemize} 
Although the principle of maximum conditional entropy characterizes the solution to \eqref{general minimax}, computing the maximizing distribution is hard in general. In \cite{minMI}, the authors propose a conditional version of the principle of maximum entropy, for the specific case of Shannon entropy, and draw the principle's connection to \eqref{general minimax}. They call it the principle of minimum mutual information,
by which one should predict based on the distribution minimizing mutual information among $\mathbf{X}$ and $Y$. 
However, they develop their theory targeting a broad class of distribution sets, which results in a convex problem, yet the number of variables is exponential in the dimension of the problem.

To overcome this issue, we propose a specific structure for the distribution set by matching the marginal $P_{\mathbf{X}}$ of all the joint distributions $P_{\mathbf{X,Y}}$ in $\Gamma(\hat{P})$ to the empirical marginal $\hat{P}_{\mathbf{X}}$ while matching only the cross-moments between $\mathbf{X}$ and $Y$ with those of the empirical distribution $\hat{P}_{\mathbf{X,Y}}$.  We show that this choice of $\Gamma(\hat{P})$ has two key advantages: 1) the minimax decision rule $\psi^*$ can be computed efficiently; 2) the minimax generalization error can be controlled by allowing a level of uncertainty in the matching of the cross-moments, which can be viewed as regularization in the minimax framework.

More importantly, by applying this idea for the generalized conditional entropy we generalize the duality shown in \cite{berger} among the maximum conditional Shannon entropy problem and the maximum likelihood problem for fitting the logistic regression model. 
%
In particular, we show how under quadratic and logarithmic loss functions our framework leads to the linear regression and logistic regression models respectively. Through the same framework, we also derive a classifier which we call the \emph{maximum entropy machine (MEM)}. We also show how regularization in the empirical risk minimization problem can be interpreted as expansion of the uncertainty set in the dual maximum conditional entropy problem, which allows us to bound the generalization worst-case error in the minimax framework.
\section{Two Examples}
In this section, we highlight two important examples to compare the minimax approach with the ERM approach. These examples also motivate a particular structure for the distribution set in the minimax approach, which is discussed earlier in the introduction.
\subsection{Regression: Squared-error}
Consider a regression task to predict a continuous $Y\in\mathbb{R}$ from feature vector $\mathbf{X}\in \mathbb{R}^d$. A well-known approach for this task is the linear regression, where one considers the set of linear prediction rules $
\{ \psi:\: \exists \boldsymbol{\beta}\in\mathbb{R}^d,\: \forall\mathbf{x}\in\mathbb{R}^d:\, \psi(\mathbf{x})=\boldsymbol{\beta}^T\mathbf{x}\}$. The ERM problem over this function class is the least-squares problem, where given samples $(\mathbf{x}_i,y_i)_{i=1}^n$ we solve
\begin{equation}
\min_{\boldsymbol{\beta}}\: \frac{1}{n}\sum_{i=1}^n \bigl( y_i-\boldsymbol{\beta}^T\mathbf{x}_i \bigr)^2.
\end{equation}
Interestingly, the minimax approach for the squared-error loss also results in the linear regression and least-squares method. Consider the space of functions $\Psi=\{ \psi:\mathbb{R}^d\rightarrow \mathbb{R}\}$ and define $\Gamma$ as the following set of distributions fixing the cross-moments $\mathbf{E}[Y\mathbf{X}]$ and the $P_\mathbf{X}$ marginal using the data
\begin{equation} \label{Gamma: Definition, squared-error}
\Gamma =\biggl\{\, P_{\mathbf{X},Y}:\,\; P_{\mathbf{X}}=\hat{P}_{\mathbf{X}},\; \mathbb{E}[Y\mathbf{X}]= \frac{1}{n}\sum_{i=1}^n y_i\mathbf{x}_i, \; \mathbb{E}[Y^2]=\frac{1}{n}\sum_{i=1}^n y_i^2\,\biggr\},
\end{equation}
where $\hat{P}_{\mathbf{X}}$ is the empirical marginal $P_{\mathbf{X}}$. Then, in Section 4 we show if we solve the minimax problem
\begin{equation} \label{minimax: linear regression}
\min_{\psi \in \Psi}\: \max_{P\in \Gamma}\: \mathbb{E}\left[ \, \bigl(Y- \psi(\mathbf{X})\bigr)^2 \,\right]
\end{equation}
the minimax optimal $\psi^*$ is a linear function which is the same as the solution to the least-squares problem. This simple example motivates the minimax approach and $\Gamma$ defined above by fixing the cross-moments and $P_\mathbf{X}$ marginal. Note that the maximin problem corresponding to \eqref{minimax: linear regression} is
\begin{equation} \label{maximin: linear regression}
 \max_{P\in \Gamma}\: \min_{\psi \in \Psi}\: \mathbb{E}\left[ \, \bigl(Y- \psi(\mathbf{X})\bigr)^2 \,\right] \: = \: \max_{P\in \Gamma}\: \mathbb{E}\bigl[\, \text{\rm Var}(Y|\mathbf{X})\, \bigr],
\end{equation}
where $\mathbb{E}[ \text{\rm Var}(Y|\mathbf{X})]$ is in fact the generalized conditional entropy for the squared-error loss function.
\subsection{Classification: 0-1 Loss}
0-1 loss is a loss function of central interest for the classification task. In the binary classification problem, the empirical risk minimization problem over linear decision rules is commonly formulated as
\begin{equation}
\min_{\boldsymbol{\beta}}\: \frac{1}{n}\sum_{i=1}^n\, \mathbf{1}\bigl(\, y_i\boldsymbol{\beta}^T\mathbf{x}_i \le 0\, \bigr)
\end{equation}
where $\mathbf{1}$ denotes the indicator function. This ERM problem to minimzie the number of missclassifications over the training samples is known to be non-convex and NP-hard \cite{NP_01}. To resolve this issue in practice, the 0-1 loss is replaced with a surrogate loss function. The hinge loss is an important example of a surrogate loss function which is empirically minimized by the Support Vector Machine (SVM) \cite{SVM} and is defined for a binary $Y\in \{-1,+1 \}$ as
\begin{equation}
\ell_{\text{\rm \scriptsize hinge}}\bigl(\, y\, , \,\boldsymbol{\beta}^T\mathbf{x}\,\bigr)=\max\bigl\lbrace \, 0\, , \, 1 -\boldsymbol{\beta}^T\mathbf{x}y \, \bigr\rbrace.
\end{equation}

On the other hand, one can change the loss function from squared-error to 0-1 loss and solve the minimax problem \eqref{minimax: linear regression} instead of empirical risk minimization. Then, by swapping the order of min and max as
\begin{equation}
\min_{\psi \in \Psi}\: \max_{P\in \Gamma}\: \mathbb{E}\left[ \, L_{\text{\rm 0-1}}(Y , \psi(\mathbf{X}) ) \,\right] \, =\,  \max_{P\in \Gamma}\: \min_{\psi \in \Psi}\:  \mathbb{E}\left[ \, L_{\text{\rm 0-1}}(Y , \psi(\mathbf{X}) ) \,\right] \, = \,  \max_{P\in \Gamma}\: H_{\text{\rm 0-1}}(Y|\mathbf{X}),
\end{equation}
we reduce the 0-1 loss minimax problem to the maximization of a concave objective $H_{\text{\rm 0-1}}(Y|\mathbf{X})$ over a convex set of probability distributions $\Gamma$. Therefore, unlike the ERM problem with 0-1 loss, this minimax problem can be solved efficiently by the MEM method. In fact, MEM solves the maximum conditional entropy problem by reducing it to a convex ERM problem. For a binary $Y\in \{-1,+1\}$, the new ERM problem has a loss function to which we call the \emph{minimax hinge loss} defined as
\begin{equation}
\ell_{\text{\rm \scriptsize mmhinge}}\bigl(\, y\, , \,\boldsymbol{\beta}^T\mathbf{x}\,\bigr)=\max\biggl\lbrace \, 0\, , \, \frac{1-\boldsymbol{\beta}^T\mathbf{x}y}{2} \, ,\, -\boldsymbol{\beta}^T\mathbf{x}y \biggr\rbrace.
\end{equation}
\begin{figure}[t]
  \centering 
  \includegraphics[width=0.8\textwidth]{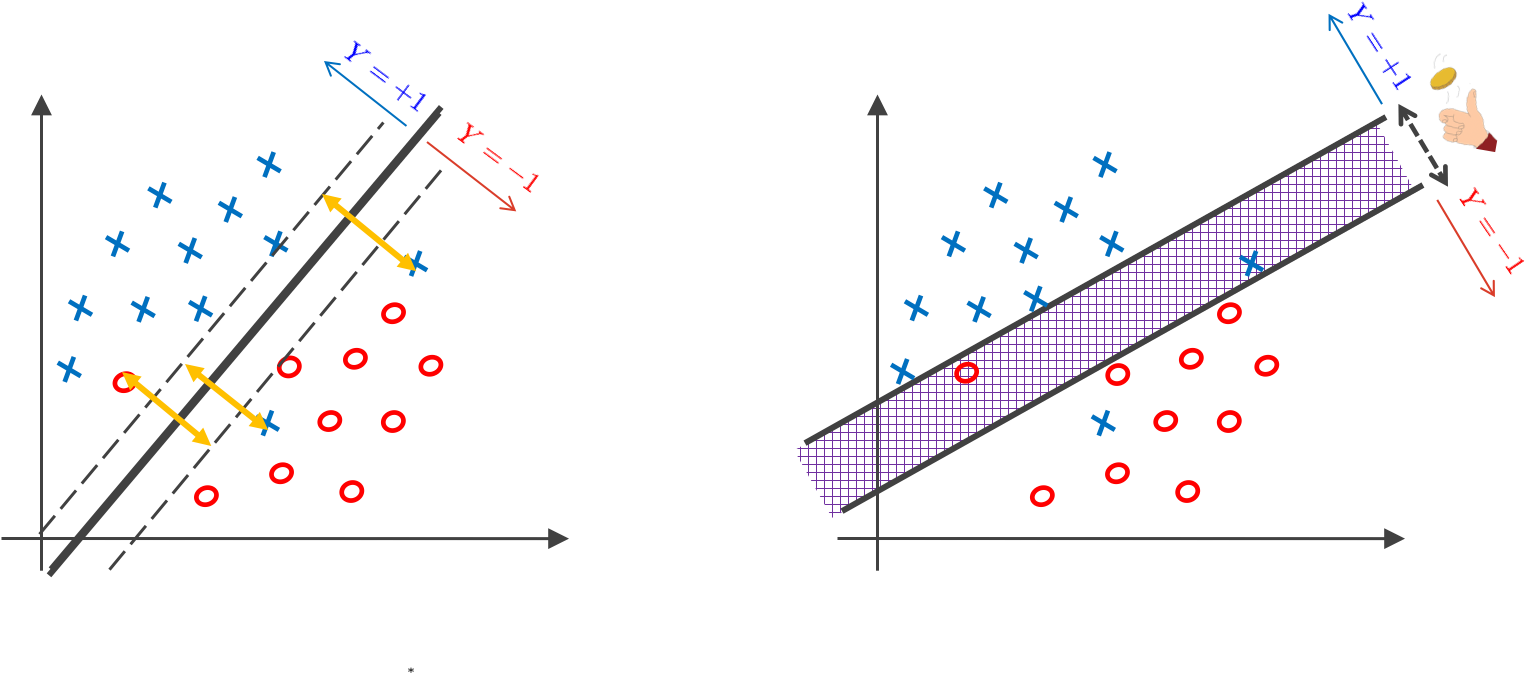}
  \caption{The determinstic linear prediction rule for SVM in the ERM approach (left picture) vs. the randomized linear prediction rule for MEM in the minimax approach (right picture)}    \label{Fig: Minimax hinge vs. Hinge}
\end{figure}
As seen in Figure \ref{Fig:DHL}, the minimax hinge loss is different from the hinge loss, and while the hinge loss is an adhoc surrogate loss function, the minimax hinge loss emerges naturally from the minimax framework. Another notable difference between the ERM and minimax frameworks is that while the linear prediction rule coming from the ERM framework is deterministic, the prediction rule resulted from the minimax approach is randomized linear (See Figure \ref{Fig: Minimax hinge vs. Hinge}). Indeed, the relaxation from the deterministic rules to the randomized rules is an important step to overcome the computational intractability of 0-1 loss minimization problem in the minimax approach. We will discuss the details of the randomized prediction rule for MEM later in Section 4. Therefore, 0-1 loss provides an important example where, unlike the ERM problem, the generalized maximum entropy framework developed at \cite{MaxEnt} results in a computationally tractable problem which is well-connected to the loss function.

\section{Principle of Maximum Conditional Entropy}
In this section, we provide a conditional version of the key definitions and results developed in \cite{MaxEnt}. We propose the principle of maximum conditional entropy to break Step 3 into 3a and 3b in Figure 1. We also define and characterize Bayes decision rules for different loss functions to address Step 3b.
\subsection{Decision Problems, Bayes Decision Rules, Conditional Entropy} \label{Subsection: CE,MI,Div}
Consider a decision problem. Here the decision maker observes $X\in \mathcal{X}$ from which she predicts a random target variable  $Y\in \mathcal{Y}$ using an action $a \in \A$. 
Let $P_{X,Y}=(P_X,P_{Y\vert X})$ be the underlying distribution for the random pair $(X,Y)$.   
Given a loss function $L: \mathcal{Y}\times\mathcal{A} \rightarrow [0,\infty]$, $L(y,a)$ indicates the loss suffered by the decision maker by deciding action $a$ when $Y=y$. 
The decision maker uses a decision rule $\psi: \mathcal{X} \rightarrow \mathcal{A}$ to select an action $a= \psi(x)$ from $\mathcal{A}$ based on an observation $x \in \X$. We  will in general allow the decision rules to be random, i.e. $\psi$ is random. The main purpose of extending to  the space of randomized decision rules is to form a convex set of decision rules. Later in Theorem \ref{MaxEntThm}, this convexity is used to prove a saddle-point theorem.

We call a (randomized) decision rule $\psi_{\text{\rm Bayes}}$ a Bayes decision rule if for all decision rules $\psi$ and for all $x \in \X$:
\begin{equation*}
 \E[L(Y, \psi_{\text{\rm Bayes}} (X)) | X = x] \le  \E[L(Y, \psi (X)) | X = x].
\end{equation*}   
It should be noted that $\psi_{\text{\rm Bayes}}$ depends only on $P_{Y\vert X}$, i.e. it remains a Bayes decision rule under a different $P_X$. The (unconditional) entropy of $Y$ is defined as \cite{MaxEnt}
\begin{equation}
H(Y) := \:\inf_{a \in \A} \: \E[L(Y, a)].
\end{equation}
Similarly, we can define conditional entropy of $Y$ given $X=x$ as
\begin{equation}
H(Y\vert X=x) :=\, \inf_{\psi}\; \E[L(Y, \psi (X)) | X = x],
\end{equation}
and the conditional entropy of $Y$ given $X$ as
\begin{equation}
H(Y\vert X) := \sum_x P_X(x) H(Y\vert X=x) =\, \inf_{\psi}\; \E[L(Y, \psi (X)) ].
\end{equation}

Note that $H(Y \vert X=x)$ and $H(Y\vert\, X) $ are both concave in $P_{Y\vert X}$. Applying Jensen's inequality, this concavity implies that
\begin{align*}
H(Y|X) \le H(Y),
\end{align*} 
which motivates the following definition for the information that $X$ carries about $Y$,
\begin{equation} \label{Def: Information}
I(X;Y) := H(Y) - H(Y\vert X) ,
\end{equation}
i.e. the reduction of expected loss in predicting $Y$ by observing $X$. In \cite{Dawid}, the author has defined the same concept to which he calls a coherent dependence measure. It can be seen that
\[I(X;Y)=\E_{P_X}[\, D(P_{Y|X},P_Y)\,]\] 
where $D$ is the divergence measure corresponding to the loss $L$, defined for any two probability distributions $P_Y,\, Q_Y$ with Bayes actions $a_P,\, a_Q$ as \cite{MaxEnt}
\begin{equation} \label{Eq: Divergence}
D(P_Y,Q_Y):= E_P [L(Y,a_Q)] - E_P [L(Y,a_P)] = E_P [L(Y,a_Q)] - H_P(Y).
\end{equation}  

\subsection{Examples}\label{subsec: examples}

\subsubsection{Logarithmic loss}\label{subsec: LogLoss}
For an outcome $y\in \mathcal{Y}$ and distribution $Q_{Y}$, define logarithmic loss as $L_{\log}(y,Q_{Y})=-\log Q_Y(y)$. 
It can be seen $H_{\log}(Y)$, $H_{\log}(Y \vert X)$, $I_{\log}(X; Y)$
are the well-known unconditional, conditional Shannon entropy and mutual information \cite{cover}. Also, the Bayes decision rule for a distribution $P_{X,Y}$ is given by $\psi_{\text{\rm Bayes}}(x) = P_{Y|X}( \cdot | x)$.
\subsubsection{0-1 loss}\label{subsec: 0-1Loss}
The 0-1 loss function is defined for any $y,\hat{y}\in \mathcal{Y}$ as $L_{\text{\rm 0-1}}(y,\hat{y})=\mathbf{1}(\hat{y} \neq y)$.
Then, we can show
\begin{equation*}
 H_{\text{\rm 0-1}}(Y) = 1- \max_{y \in \mathcal{Y}} P_{Y}(y), \quad H_{\text{\rm 0-1}}(Y\vert X) =1 -\sum_{x\in \mathcal{X}} \max_{y \in \mathcal{Y}} P_{X,Y}(x,y).
\end{equation*}
The Bayes decision rule for a distribution $P_{X,Y}$ is the well-known maximum a posteriori (MAP) rule, i.e. $\psi_{\text{\rm Bayes}}(x)= {\arg\!\max}_{y\in \mathcal{Y}}\: P_{Y\vert X}(y\vert x). $
\subsubsection{Quadratic loss}\label{subsec: QuadraticLoss}
The quadratic loss function is defined as $L_{2}(y,\hat{y})= (y- \hat{y} )^{2}$.
It can be seen 
\begin{equation*}
H_{2}(Y) = \text{Var}(Y),\quad H_{2}(Y\vert X) = \mathbb{E}\:[\text{Var}(Y\vert X)], \quad I_2(X ; Y) =  \text{Var}\left(\mathbb{E}[Y\vert X]\right).
\end{equation*}
The Bayes decision rule for any $P_{X,Y}$ is the well-known minimum mean-square error (MMSE) estimator that is $\psi_{\text{\rm Bayes}}(x)=\mathbb{E}[Y\vert X=x] $.

\subsection{Principle of Maximum Conditional Entropy \& Robust Bayes decision rules}
Given a distribution set $\Gamma$, consider the following minimax problem to find a decision rule minimizing the worst-case expected loss over $\Gamma$
\begin{equation} \label{GeneralMinimaxProblem}
\underset{\psi \in \boldsymbol{\Psi}}{\arg\!\min}\;\max_{P\in \Gamma}\, \E_P[L(Y, \psi (X))], 
\end{equation}
where $\boldsymbol{\Psi}$ is the space of all randomized mappings from $\mathcal{X}$ to $\mathcal{A}$ and $\E_P$ denotes the expected value over distribution $P$. 
We call any solution $\psi^{*}$ to the above problem a robust Bayes decision rule against $\Gamma$. The following results motivate a generalization of the maximum entropy principle to find a robust Bayes decision rule. Refer to the Appendix for the proofs.
\begin{subtheorem}{thm}\label{MaxEntThm}
\begin{thm}\label{thm:oneA}
(Weak Version) Suppose $\Gamma$ is convex and closed, and let $L$ be a bounded loss function. Assume $\mathcal{X},\mathcal{Y}$ are finite and that the risk set $S=\lbrace \, \left[ L(y,a)\right]_{y\in \mathcal{Y}} \, :\, a\in \mathcal{A}\,\rbrace $ is closed.  
Then there exists a robust Bayes decision rule $\psi^*$ against $\Gamma$, which is a Bayes decision rule for a distribution $P^*$ that maximizes the conditional entropy $H(Y|X)$ over $\Gamma$. 
\end{thm}
\begin{thm}\label{thm:oneB}
(Strong Version) Suppose $\Gamma$ is convex and that under any $P\in \Gamma$ there exists a Bayes decision rule. We also assume the continuity in Bayes decision rules for distributions in $\Gamma$ (See the Appendix for the exact condition). Then, if $P^{*}$ maximizes $H(Y|X)$ over $\Gamma$, any Bayes decision rule for $P^{*}$ is a robust Bayes decision rule against $\Gamma$. 
\end{thm}
\end{subtheorem}
\textbf{Principle of Maximum Conditional Entropy}: Given a set of distributions $\Gamma$, predict $Y$ based on a distribution in $\Gamma$ that maximizes the conditional entropy of $Y$ given $X$, i.e. 
\begin{equation}  \label{MaxEntProblem}
\underset{P\in \Gamma}{\arg\!\max}\; H(Y\vert X)
\end{equation} 

Note that while the weak version of Theorem \ref{MaxEntThm} guarantees \textbf{only the existence} of a saddle point for \eqref{GeneralMinimaxProblem}, the strong version further guarantees that \textbf{any} Bayes decision rule of the maximizing distribution results in a robust Bayes decision rule. However, the continuity in Bayes decision rules does not hold for the discontinuous 0-1 loss, which requires considering the weak version of Theorem \ref{MaxEntThm} to address this issue. 

\section{Prediction via Maximum Conditional Entropy Principle}
Consider a prediction task with target variable $Y$ and feature vector $\mathbf{X}=(X_1,\ldots ,X_d)$. We do not require the variables to be discrete. As discussed earlier, the maximum conditional entropy principle reduces \eqref{GeneralMinimaxProblem} to \eqref{MaxEntProblem}, which formulate steps 3 and 3a in Figure \ref{Fig: MCE}, respectively.
However, a general formulation of \eqref{MaxEntProblem} in terms of the joint distribution $P_{\mathbf{X},Y}$ leads to an exponential computational complexity in the feature dimension $d$. 

The key question is therefore under what structures of $\Gamma(\hat{P})$ in Step 2 we can solve \eqref{MaxEntProblem} efficiently. In this section, we propose a specific structure for $\Gamma(\hat{P})$, under which we provide an efficient solution to Steps 3a and 3b in Figure 1. In addition, we prove a bound on the generalization worst-case risk for the proposed $\Gamma(\hat{P})$. In fact, we derive these results by reducing 
\eqref{MaxEntProblem} to the maximum likelihood problem over a generalized linear model, under this specific structure.

To describe this structure, consider a set of distributions $\Gamma(Q)$ centered around a given distribution $Q_{\mathbf{X},Y}$, where for a given norm $\Vert\cdot\Vert$, mapping vector $\boldsymbol{\theta}(Y)_{t\times 1}$,
\begin{align}
\Gamma(Q)=\lbrace \: P_{\mathbf{X},Y}:& \: P_\mathbf{X}=Q_\mathbf{X} \, , \label{Eq: Classform} \\
&\: \forall\, 1\le i\le t:\;\;  \Vert\, \mathbb{E}_{P}\left[{\theta}_i(Y)\mathbf{X}\right] - \mathbb{E}_{Q}\left[ {\theta}_i(Y)\mathbf{X}\right]\Vert \le \epsilon_i \:\rbrace . \nonumber
\end{align}
Here $\boldsymbol{\theta}$ encodes $Y$ with $t$-dimensional $\boldsymbol{\theta}(Y)$, and ${\theta}_i(Y)$ denotes the $i$th entry of $\boldsymbol{\theta}(Y)$.  The first constraint in the definition of $\Gamma(Q)$ requires all distributions in $\Gamma(Q)$ to share the same marginal on $\mathbf{X}$ as $Q$; the second  imposes constraints on the cross-moments between $\mathbf{X}$ and $Y$, allowing for some uncertainty in estimation. When applied to the supervised learning problem, we will choose $Q$ to be the empirical distribution $\hat{P}$ and select $\boldsymbol{\theta}$ appropriately based on the loss function $L$. However, for now we will consider the problem of solving \eqref{MaxEntProblem} over $\Gamma (Q)$ for general $Q$ and $\boldsymbol{\theta}$.

To that end, we use a similar technique as in the Fenchel's duality theorem, also used at \cite{altun2006,dudik,semi2010} to address divergence minimization problems. However, we consider a different version of convex conjugate for $-H$, which is defined with respect to $\boldsymbol{\theta}$. Considering $\mathcal{P_Y}$ as the set of all probability distributions for the variable $Y$, we define $F_{\boldsymbol{\theta}}:\, \mathbb{R}^t \rightarrow \mathbb{R} $ as the convex conjugate of $-H(Y)$ with respect to the mapping $\boldsymbol{\theta}$,
\begin{equation}\label{F: definition}
F_{\boldsymbol{\theta}}(\mathbf{z})\, := \, \max_{P\in \mathcal{P_Y}}\,  H(Y) + \mathbb{E}[\boldsymbol{\theta}(Y)]^T\mathbf{z}. 
\end{equation}  
\begin{thm} \label{Thm: duality}
Define $\Gamma(Q)$, $F_{\boldsymbol{\theta}}$ as given by \eqref{Eq: Classform}, \eqref{F: definition}. Then the following duality holds
\begin{equation} \label{Eq: duality}
\max_{P\in\Gamma(Q)} H(Y\vert \mathbf{X}) \: =\min_{\mathbf{A \in \mathbb{R}^{t\times d}}}\: \mathbb{E}_{Q} \left[\, F_{\boldsymbol{\theta}} (\mathbf{A}\mathbf{X}) - \boldsymbol{\theta}(Y)^T\mathbf{A}\mathbf{X}\, \right] + \sum_{i=1}^t {\epsilon_i} \Vert \mathbf{A}_i \Vert_{*},
\end{equation}
where $\Vert \mathbf{A}_i \Vert_{*}$ denotes $\Vert\cdot \Vert$'s dual norm of the $\mathbf{A}$'s $i$th row. Furthermore, for the optimal $P^*$ and $\mathbf{A}^*$
\begin{equation} \label{Eq: duality_KKT}
\mathbb{E}_{P^*} [\,\boldsymbol{\theta}(Y) \,\vert\, \mathbf{X}=\mathbf{x}\,] = \nabla  F_{\boldsymbol{\theta}} \, (\mathbf{A}^*\mathbf{x}).
\end{equation}
\end{thm}
\begin{proof}
Refer to the the supplementary material for the proof.
\end{proof}
When applying Theorem \ref{Thm: duality} on a supervised learning problem with a specific loss function, $\theta$ will be chosen such that $\mathbb{E}_{P^*} [\,\boldsymbol{\theta}(Y) \,\vert\, \mathbf{X}=\mathbf{x}\,] $ provides sufficient information to compute the Bayes decision rule $\Psi^*$ for $P^*$. 
This enables the direct computation of $\psi^*$, i.e. step 3 of Figure \ref{Fig: MCE}, without the need to explicitly compute $P^*$ itself.  For the loss functions discussed at Subsection \ref{subsec: examples}, we choose the identity $\boldsymbol{\theta}(Y)=Y$ for the quadratic loss and the one-hot encoding $\boldsymbol{\theta}(Y)=[\,\mathbf{1}(Y=i)\,]_{i=1}^{t}$ for the logarithmic and 0-1 loss functions. Later in this section, we will discuss how this theorem applies to these loss functions.

We make the key observation that the problem in the RHS of \eqref{Eq: duality}, when $\epsilon_i=0$ for all $i$'s, is equivalent to minimizing the negative log-likelihood for fitting a generalized linear model \cite{GLM} given by
\begin{itemize}[leftmargin=*]
\item An exponential-family distribution $p(y|\boldsymbol{\eta})=h(y)\exp\left( \boldsymbol{\eta}^T \boldsymbol{\theta}(y) - F_{\boldsymbol{\theta}}(\boldsymbol{\eta})\right)$ with the log-partition function $F_{\boldsymbol{\theta}}$ and the sufficient statistic $\boldsymbol{\theta}(Y)$,
\item A linear predictor, $\boldsymbol{\eta}(\mathbf{X})=\mathbf{A}\mathbf{X}$,
\item A mean function, $\mathbb{E}[\,\boldsymbol{\theta}(Y) \vert \mathbf{X}=\mathbf{x}] = \nabla  F_{\boldsymbol{\theta}}  (\boldsymbol{\eta}(\mathbf{x}))$.
\end{itemize} 
\begin{figure}[t]
  \centering 
  \includegraphics[width=0.4\textwidth]{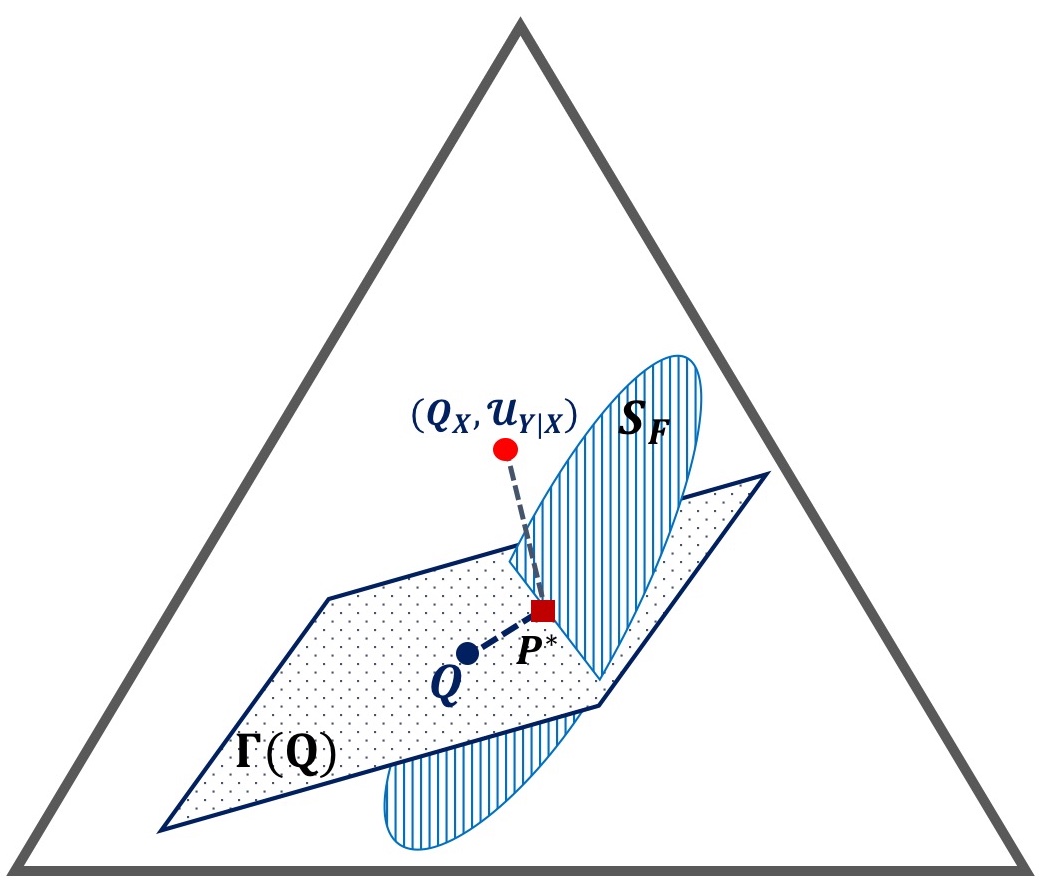}
  \caption{Duality of Maximum Conditional Entropy/Maximum Likelihood in GLMs}    \label{Fig: MaxEnt duality}
\end{figure}
Therefore, Theorem \ref{Thm: duality} reveals a duality between the maximum conditional entropy problem over $\Gamma(Q)$ and the regularized maximum likelihood problem for the specified generalized linear model. This duality further provides a minimax justification for generalized linear models and fitting them using maximum likelihood, since we can consider the convex conjugate of its log-partition function as the negative entropy in the maximum conditional entropy problem.
\subsection{Generalization Bound on the Worst-case Risk}
By establishing the objective's Lipschitzness and boundedness through appropriate assumptions, we can apply standard results to bound the rate of uniform convergence for the problem in the RHS of \eqref{Eq: duality}. Here we consider the uniform convergence of the empirical averages, when $Q=\hat{P}_n$ is the empirical distribution of $n$ samples drawn i.i.d. from the underlying distribution $\tilde{P}$, to their expectations when $Q=\tilde{P}$.   

In the supplementary material, we prove the following theorem which bounds the generalization worst-case risk, by interpreting the mentioned uniform convergence on the other side of the duality. Here ${{\hat{\psi}}_{n}}$ and $\tilde{\psi}$ denote the robust Bayes decision rules against $\Gamma(\hat{P}_n)$ and $\Gamma(\tilde{P})$, respectively. As explained earlier, by the maximum conditional entropy principle we can learn ${{\hat{\psi}}_{n}}$ by solving the RHS of \eqref{Eq: duality} for the empirical distribution and then applying \eqref{Eq: duality_KKT}.
\begin{thm} \label{Thm: Learnability 1}
Consider a loss function $L$ with the entropy function $H$ and suppose $\boldsymbol{\theta}(Y)$ includes only one element, i.e. $t=1$. Let $M=\max_{P\in\mathcal{P_Y}}\, H(Y)$ be the maximum entropy value over $\mathcal{P_Y}$. Also, take $\Vert\cdot \Vert/ \Vert\cdot \Vert_*$ to be the $\ell_p/\ell_q$ pair where $\frac{1}{p}+\frac{1}{q}=1$, $1 \le q\le 2$. 
Given that $\Vert \mathbf{X}\Vert_2 \le B$ and $\vert {\theta}(Y)\vert \le L$, for any $\delta>0$ with probability at least $1-\delta$
\begin{equation}
\max_{P\in \Gamma(\tilde{P})}\E[ L(Y,\hat{\psi}_n(\mathbf{X}))] \, -\, \max_{P\in \Gamma(\tilde{P})} \E [ L(Y,\tilde{\psi}(\mathbf{X}))] \: \le \: \frac{8BLM}{\epsilon\sqrt{n}}\,\biggl( 1+ \sqrt{\frac{\log(2/\delta)}{2}} \biggr). 
\end{equation}
\end{thm}  
Theorem \ref{Thm: Learnability 1} states that though we learn the prediction rule $\hat{\psi}_n$ by solving the maximum conditional problem for the empirical case, we can bound the excess $\Gamma$-based worst-case risk. This generalization result justifies the constraint of fixing the marginal $P_\mathbf{X}$ across the proposed $\Gamma(Q)$ and explains the role of the uncertainty parameter $\epsilon$ in bounding the generalization worst-case risk. 
\subsection{Geometric Interpretation of Theorem \ref{Thm: duality}}
By solving the regularized maximum likelihood problem in the RHS of \eqref{Eq: duality}, we in fact minimize a regularized KL-divergence 
\begin{equation} \label{Eq: Duality 2}
\underset{P_{Y|\mathbf{X}}\in S_{F}}{\arg\!\min}\;\;\mathbb{E}_{Q_\mathbf{X}}[\,D_{\text{\rm KL}}(\, Q_{Y|\mathbf{X}}\, || \, P_{Y|\mathbf{X}}\, )\,] + \sum_{i=1}^t {\epsilon_i} \Vert \mathbf{A}_i(P_{Y|\mathbf{X}}) \Vert_{*},
\end{equation}
where $S_{F}=\lbrace P_{Y|\mathbf{X}}(y\vert\mathbf{x})=h(y)\exp(\, \boldsymbol{\theta}(y)^T\mathbf{A}\mathbf{x}  - F_{\boldsymbol{\theta}}(\mathbf{A}\mathbf{x})\,  | \, \mathbf{A}\in \mathbb{R}^{t\times s} \rbrace$ is the set of all exponential-family conditional distributions for the specified generalized linear model. This can be viewed as projecting $Q_{Y|X}$ onto $S_{F}$ (See Figure \ref{Fig: MaxEnt duality}). 

Furthermore, it can be seen that for a label-invariant entropy function $H(Y)$, the Bayes act for the uniform distribution $\mathcal{U}_{Y}$ leads to the same expected loss under any distribution $P_Y$ on $Y$. Based on the divergence $D$'s definition in \eqref{Eq: Divergence}, maximizing $H(Y|\mathbf{X})$ over $\Gamma(Q)$ in the LHS of \eqref{Eq: duality} is therefore equivalent to the following divergence minimization problem
\begin{equation} \label{Eq: Duality 1}
\underset{P_{Y|\mathbf{X}}:\: (Q_\mathbf{X},P_{Y|\mathbf{X}})\in \Gamma(Q)}{\arg\!\min}\;\;\mathbb{E}_{Q_\mathbf{X}}[\,D(P_{Y|\mathbf{X}},\mathcal{U}_{Y|\mathbf{X}})\,].
\end{equation} 
Here $\mathcal{U}_{Y|\mathbf{X}}$ denotes the uniform conditional distribution over $Y$ given any $x\in\mathcal{X}$. This can be interpreted as projecting the joint distribution $(Q_\mathbf{X},\mathcal{U}_{Y|\mathbf{X}})$ onto $\Gamma(Q)$ (See Figure \ref{Fig: MaxEnt duality}). Then, the duality shown in Theorem \ref{Thm: duality} implies the following corollary.
\begin{cor}
The solution to \eqref{Eq: Duality 2} would also minimize \eqref{Eq: Duality 1}, i.e. \eqref{Eq: Duality 2} $\subseteq$ \eqref{Eq: Duality 1}.
\end{cor} 

\subsection{Examples}
\subsubsection{Logarithmic Loss: Logistic Regression} \label{Subsection: logistic regression}
To gain sufficient information for the Bayes decision rule under the logarithmic loss, for $Y\in\mathcal{Y}=\lbrace 1,\ldots ,t+1 \rbrace$, let $\boldsymbol{\theta}(Y)$ be the one-hot encoding of $Y$, i.e. $\boldsymbol{\theta}_i(Y)=\mathbf{1}(Y=i)$ for $1\le i\le t$. Here, we exclude $i=t+1$ as $\mathbf{1}(Y=t+1)=1-\sum_{i=1}^{t}\mathbf{1}(Y=i)$. Then
\begin{equation}
F_{\boldsymbol{\theta}}(\mathbf{z})=\log \bigl( 1+ \sum_{j=1}^t \exp(\mathbf{z}_j)\bigr),\quad  \forall\: 1\le i\le t:\;\: \bigl({\nabla F_{\boldsymbol{\theta}}}(\mathbf{z})\bigr)_i  =\exp\left({\mathbf{z}_i}\right) / \bigl( 1+ \sum_{j=1}^{t}\exp (\mathbf{z}_j)\bigr),
\end{equation}
which is the logistic regression model \cite{elements}. Also, the RHS of \eqref{Eq: duality} will be the regularized maximum likelihood problem for logistic regression. This particular result is well-studied in the literature and straightforward using the duality shown in \cite{berger}.
\subsubsection{0-1 Loss: maximum entropy machine}
To get sufficient information for the Bayes decision rule under the 0-1 loss, we again consider the one-hot encoding $\boldsymbol{\theta}$ described for the logarithmic loss. We show in the Appendix that if $\tilde{\mathbf{z}}=(\mathbf{z},0)$ and $\tilde{z}_{(i)}$ denotes the $i$th largest element of $\tilde{\mathbf{z}}$,
\begin{equation}
F_{\boldsymbol{\theta}}(\mathbf{z}) = \max_{1\le k\le t+1}\; \frac{k-1+\sum_{j=1}^k \tilde{z}_{(j)}}{k}.
\end{equation}
In particular, if $Y$ is binary where $t=1$
\begin{equation}
F_{\boldsymbol{\theta}}(z)=\max \lbrace \: 0\, ,\, \frac{z+1}{2}\, ,\, z \:\rbrace.
\end{equation}
Then, if $Y\in \mathcal{Y}=\lbrace -1,1\rbrace$
 the maximum likelihood problem \eqref{Eq: duality} for learning the optimal linear predictor $\boldsymbol{\alpha}^*$ given $n$ samples $(\mathbf{x}_i,y_i)_{i=1}^n$ will be
\begin{equation} \label{Eq: maximum entropy machine}
\min_{\boldsymbol{\alpha}}\;\: \frac{1}{n}\sum_{i=1}^n\max\biggl\lbrace 0\, , \frac{1-y_i\boldsymbol{\alpha}^T\mathbf{x}_i}{2}\, ,\, -y_i\boldsymbol{\alpha}^T\mathbf{x}_i \biggr\rbrace + \epsilon \Vert \boldsymbol{\alpha} \Vert_*.
\end{equation}
The first term is the empirical risk of a linear classifier over the minimax-hinge loss $\max \lbrace 0, \frac{1-z}{2} , -z \rbrace $ as shown in Figure \ref{Fig:DHL}.
In contrast, the standard SVM is formulated using the hinge loss $\max \lbrace 0, 1-z \rbrace$:
\begin{equation} \label{Eq: standard SVM}
\min_{\boldsymbol{\alpha}}\;\: \frac{1}{n}\sum_{i=1}^n\max\bigl\lbrace 0\, , {1-y_i\boldsymbol{\alpha}^T\mathbf{x}_i} \bigr\rbrace + \epsilon \Vert \boldsymbol{\alpha} \Vert_*,
\end{equation}
We therefore call this classification approach the maximum entropy machine. However, unlike the standard SVM, the maximum entropy machine is naturally extended to multi-class classification. 

Using Theorem 1.A\footnote{We show that given the specific structure of $\Gamma(Q)$ Theorem 1.A holds whether $\mathcal{X}$ is finite or  infinite.}, we prove that for 0-1 loss the robust Bayes decision rule exists and is randomized in general, where given the optimal linear predictor $\tilde{\mathbf{z}}=(\mathbf{A}^*\mathbf{x},0)$ randomly predicts a label according to the following $\tilde{\mathbf{z}}$-based distribution on labels
\begin{equation} \label{maximum entropy machine: multilabel f}
\forall\: 1\le i\le t+1:\; \;\; p_{\sigma(i)} = \begin{cases} \begin{split}
& \tilde{z}_{(i)} + \frac{1 - \sum_{j=1}^{k_{\max}} \tilde{z}_{(j)}}{k_{\max}} \quad \text{\rm if}\: \sigma(i)\le k_{\max},  \\
& 0 \qquad \qquad \qquad \qquad \quad \;\;\, \text{\rm Otherwise.}
\end{split}
\end{cases}
\end{equation}
Here $\sigma$ is the permutation sorting $\tilde{\mathbf{z}}$ in the ascending order, i.e. $\tilde{z}_{\sigma(i)}=\tilde{z}_{(i)}$, and $k_{\max}$ is the largest index $k$ satisfying $\sum_{i=1}^{k} [ \tilde{\mathbf{z}}_{(i)} - \tilde{\mathbf{z}}_{(k)}\, ] < 1$. For example, in the binary case discussed, the maximum entropy machine first solves \eqref{Eq: maximum entropy machine} to  find the optimal $\boldsymbol{\alpha}^*$ and then predicts label $y=1$ vs. label $y=-1$ with probability $\min \bigl\lbrace 1\, , \, \max \lbrace 0\, , \, (1+\mathbf{x}^T\boldsymbol{\alpha}^*)/2 \rbrace \bigr\rbrace$.

We can also find the conditional-entropy maximizing distribution via \eqref{Eq: duality_KKT}, where the gradient of $F_{\boldsymbol{\theta}}$ is given by
\begin{equation} \label{maximum entropy machine: multilabel f_3}
\forall\: 1\le i\le t:\; \;\; \bigl({\nabla F_{\boldsymbol{\theta}}}(\mathbf{z})\bigr)_i = \begin{cases} \begin{split}
&1/k_{\max} \quad \text{\rm if}\: \sigma(i)\le k_{\max},  \\
& 0 \qquad \quad\;\; \text{\rm Otherwise.}
\end{split}
\end{cases}
\end{equation}
 Note that $F_\theta$ is not differentiable if $\sum_{i=1}^{k_{\max}} [ \tilde{\mathbf{z}}_{(i)} - \tilde{\mathbf{z}}_{(k_{\max}+1)}\, ] = 1$, but the above vector is still in the subgradient $\partial F_\theta (\mathbf{z})$. Although we can find the $H(Y|X)$-maximizing distribution, there could be multiple Bayes decision rules for that distribution. Since the strong result in Theorem \ref{MaxEntThm} does not hold for the 0-1 loss, we are not guaranteed that all these decision rules are robust against $\Gamma(\hat{P})$. However, as we show in the appendix the randomized decision rule given by \eqref{maximum entropy machine: multilabel f} will be robust.
 
\subsubsection{Quadratic Loss: Linear Regression}
Based on the Bayes decision rule for the quadratic loss, we choose $\boldsymbol{\theta}(Y)=Y$. To derive $F_{\boldsymbol{\theta}}$, note that if we let $\mathcal{P_Y}$ in \eqref{F: definition} include all possible distributions, the maximized entropy (variance for quadratic loss) and thus the value of $F_{\boldsymbol{\theta}}$ would be infinity. Therefore, given a parameter $\rho$, we restrict the second moment of distributions in $\mathcal{P_Y}=\lbrace P_Y:\, \E[Y^2]\le \rho^2\rbrace$ and then apply \eqref{F: definition}. We show in the Appendix that an adjusted version of Theorem \ref{Thm: duality} holds after this change, and
\begin{equation}
F_{\boldsymbol{\theta}}(z) - \rho^2=\begin{cases}
z^2/4\quad &\text{\rm if}\; |z/2|\le \rho \\
\rho ( |z| - \rho)\quad & \text{\rm if}\; |z/2| > \rho,
\end{cases}
\end{equation}
which is the Huber function \cite{huber1981}. To find $\E[Y|\mathbf{X}]$ via \eqref{Eq: duality_KKT}, we have
\begin{equation} \label{Bayes Rule: Quadratic loss}
\frac{d F_{\boldsymbol{\theta}}(z)}{dz} =\begin{cases}
-\rho \quad &\text{\rm if}\;\; z/2\le -\rho \\
z/2\quad &\text{\rm if}\; -\rho < z/2 \le \rho \\
\rho \quad & \text{\rm if}\quad\;\: \rho < z/2.
\end{cases}
\end{equation}
Given the samples of a supervised learning task if we choose the parameter $\rho$ large enough, by solving the RHS of \eqref{Eq: duality} when $F_{\boldsymbol{\theta}}(z)$ is replaced with $z^2/4$ and set $\rho$ greater than $\max_{i} |\mathbf{A^*}\mathbf{x}_i|$, we can equivalently take $F_{\boldsymbol{\theta}}(z) = z^2/4 +\rho^2$. Then, \eqref{Bayes Rule: Quadratic loss} reduces to the linear regression model and the maximum likelihood problem in the RHS of \eqref{Eq: duality} is equivalent to 
\begin{itemize}[leftmargin=*,topsep=0pt]
\item[--] Least squares when $\epsilon=0$. 
\item[--] Lasso \cite{lasso,Lasso_Donoho} when  $\Vert\cdot \Vert /\Vert\cdot \Vert_*$ is the $\ell_\infty /\ell_1$ pair.
\item[--] Ridge regression \cite{Ridge} when $\Vert\cdot \Vert $ is the $\ell_2$-norm.
\item[--] (overlapping) Group lasso \cite{grouplasso,GLoverlap} with the $\ell_{1,p}$ penalty when $\Gamma_{\text{\rm GL}}(Q)$ is defined, given subsets $I_1,\ldots I_k$ of $\lbrace 1,\ldots , d \rbrace$ and $1/p + 1/q =1$, as
\begin{align}
\Gamma_{\text{\rm GL}}(Q)=\lbrace \: P_{\mathbf{X},Y}:& \: P_\mathbf{X}=Q_\mathbf{X} \, , \label{Eq: Classform_GroupLasso} \\
&\: \forall\, 1\le j\le k:\;\;  \Vert\, \mathbb{E}_{P}\left[Y\mathbf{X}_{I_j}\right] - \mathbb{E}_{Q}\left[ Y\mathbf{X}_{I_j}\right]\Vert_{q} \le \epsilon_j \:\rbrace . \nonumber
\end{align}
\end{itemize}
See the Appendix for the proofs. Another type of minimax, but non-probabilistic, argument for the robustness of lasso-like regression algorithms can be found in  \cite{RobustLasso,RobustLassoLike}.

\section{Robust Feature Selection}
Using a minimax criterion over a set of distributions $\Gamma$, we solve the following problem to select the most informative subset of $k$ features,
\begin{equation} \label{FS: main formulation 1}
\underset{ {\vert S \vert \le k
}
}{\arg\!\min}\: \min_{\psi \in \boldsymbol{\Psi}_S}\: \max_{P \in \Gamma}\; \E_P[\, L(Y, \psi (\,\mathbf{X}_S\, ))\,]
\end{equation}
where $\mathbf{X}_S$ denotes the feature vector $\mathbf{X}$ restricted to the indices in $S$. Here, we evaluate each feature subset based on the minimum worst-case loss over $\Gamma$.  
Applying Theorem \ref{MaxEntThm}, \eqref{FS: main formulation 1} reduces to
\begin{equation} \label{FS: main formulation entropy}
\underset{ 
\vert S \vert \le k
}{\arg\!\min}\: \max_{P \in \Gamma}\; H(Y\vert \,\mathbf{X}_S\, ),
\end{equation}
which under the assumption that the marginal $H(Y)$ is fixed across all distributions in $\Gamma$ is equivalent to selecting a subset $S$ maximizing the worst-case generalized information $I(\mathbf{X}_S ;Y)$ over $\Gamma$, i.e.
\begin{equation}
\underset{ 
\vert S \vert \le k
}{\arg\!\max}\: \min_{P \in \Gamma}\; I(\mathbf{X}_S ;Y).
\end{equation}
To solve \eqref{FS: main formulation entropy} when $\Gamma=\Gamma(\hat{P}_n)$ defined at \eqref{Eq: Classform}, where $\hat{P}_n$ is the empirical distribution of samples $(\mathbf{x}_i,y_i)_{i=1}^n$, we apply the duality shown in Theorem \ref{Thm: duality} to obtain
\begin{equation} \label{FS: L0}
\underset{\mathbf{A} \in \mathbb{R}^{t\times s}:\, \Vert\mathbf{A}\Vert_{0,\infty} \le k}{\arg\!\min}\; \frac{1}{n} \sum_{i=1}^n\left[\, F_{\boldsymbol{\theta}} (\mathbf{A}\mathbf{x}_i) - \boldsymbol{\theta}(y_i)^T\mathbf{A}\mathbf{x}_i\, \right] + \sum_{i=1}^t {\epsilon_i} \Vert \mathbf{A}_i \Vert_{*}.
\end{equation} 
Here by constraining $\Vert\mathbf{A}\Vert_{0,\infty} = \Vert \bigl(\Vert \mathbf{A}^{(1)}\Vert_\infty,\ldots,\Vert \mathbf{A}^{(s)}\Vert_\infty\bigr)\Vert_0$ where $\mathbf{A}^{(i)}$ denotes the $i$th column of $\mathbf{A}$, we impose the same sparsity pattern 
across the rows of $\mathbf{A}$. Let $\Vert\cdot\Vert_*$ be the $\ell_1$-norm and relax the above problem to 
 \begin{equation} \label{FS: WOL0}
\underset{\mathbf{A} \in \mathbb{R}^{t\times s}}{\arg\!\min}\; \frac{1}{n} \sum_{i=1}^n\left[\, F_{\boldsymbol{\theta}} (\mathbf{A}\mathbf{x}_i) - \boldsymbol{\theta}(y_i)^T\mathbf{A}\mathbf{x}_i\, \right] + \sum_{i=1}^t {\epsilon_i} \Vert \mathbf{A}_i \Vert_{1}.
\end{equation} 

Note that if for the uncertainty parameters $\epsilon_i$'s, the solution $\mathbf{A}^*$ to \eqref{FS: WOL0} satisfies $\Vert\mathbf{A}^*\Vert_{0,\infty} \le k$ due to the tendency of $\ell_1$-regularization to produce sparse solutions, $\mathbf{A}^*$ is the solution to \eqref{FS: L0} as well. In addition, based on the generalization bound established in Theorem \ref{Thm: Learnability 1}, by allowing some gap we can generalize this sparse solution to \eqref{FS: main formulation 1} with $\Gamma = \Gamma(\tilde{P})$ for the underlying distribution $\tilde{P}$.

It is noteworthy that for the quadratic loss and identity $\theta$, \eqref{FS: WOL0} is the same as the lasso. Also, for the logarithmic loss and one-hot encoding $\boldsymbol{\theta}$, \eqref{FS: WOL0} is equivalent to the $\ell_1$-regularized logistic regression. Hence, the $\ell_1$-regularized logistic regression maximizes the worst-case mutual information over $\Gamma(Q)$, which seems superior to the methods maximizing a heusristic instead of the mutual information $I(\mathbf{X}_S;Y)$  \cite{pengfeature,feature2}.

\begin{table*}
\centering
 \newcolumntype{C}{>{\centering\arraybackslash}p{3.6em}}
 \renewcommand{\arraystretch}{1.1}
{
  \begin{tabular}{| C || C | C | C | C | C | C |}
    \hline
    Dataset & MEM  & SVM & DCC & MPM & TAN & DRC\\ \hline
    adult & \textbf{17} &  22 & 18   & 22 &  \textbf{17}   & \textbf{17} \\ \hline 
    credit & \textbf{12} & 16 &  14 & 13  &  17 & 13 \\ \hline
    kr-vs-kp & 7 & \textbf{3} &  10 & 5 & 7 & 5 \\ \hline
    promoters & \textbf{5}  & 9 & \textbf{5}  & 6  & 44 & 6 \\ \hline
    votes & 4 & 5 &  \textbf{3} & 4 & 8   & \textbf{3}  \\ \hline
    hepatitis & \textbf{17} & 20 &  19 & 18 & \textbf{17}   & \textbf{17}  \\ \hline
  \end{tabular}
  }
  \caption {Methods Performance (error in \%) }  \label{NumericalTable} 
  \end{table*}
\section{Numerical Experiments}
We evaluated the performance of the maximum entropy machine on six binary classification datasets from the UCI repository, compared to these five benchmarks: Support Vector Machines (SVM), Discrete Chebyshev Classifiers (DCC) \cite{DCC}, Minimax Probabilistic Machine (MPM) \cite{mpm}, Tree Augmented Naive Bayes (TAN) \cite{tan}, and Discrete R\'{e}nyi Classifiers (DRC) \cite{DRC}.
The results are summarized in Table \ref{NumericalTable} where the numbers indicate the percentage of error in the classification task. 

We implemented the maximum entropy machine by applying the gradient descent 
to \eqref{Eq: maximum entropy machine} with the regularizer $\lambda\Vert \boldsymbol{\alpha}\Vert_2^2$. We determined the value of $\lambda$ by cross validation. To determine the lambda coefficient, we used a randomly-selected 70\% of the training set for training and the rest 30\% of the training set for testing. We tested the values in $\lbrace 2^{-10},\ldots,2^{10}\rbrace$. Using the tuned lambda, we trained the algorithm over all the training set and then evaluated the error rate over the test set. We performed this procedure in 1000 Monte Carlo runs each training on 70\% of the data points and testing on the rest 30\% and averaged the results. 

As seen in the table, the maximum entropy machine results in the best performance for four of the six datasets. Also, note that except a single dataset the maximum entropy machine outperforms SVM. To compare these methods in high-dimensional problems, we ran an experiment over synthetic data with $n=200$ samples and $d=10000$ features. We generated features by i.i.d. Bernoulli with $P(X_i=1)=0.75$, and considered $y=\text{sign}(\gamma^T\mathbf{x}+z)$ where $z\sim N(0,1)$. 
Using the same approach, we evaluated 20.6\% error rate for SVM, 20.4\% error rate for DRC, 20.0\% for the MEM which shows the MEM can outperform SVM and DRC in high-dimensional settings as well.

\subsubsection*{Acknowledgments}
We are grateful to Stanford University providing a Stanford Graduate Fellowship, and the Center for Science of Information (CSoI), an NSF Science and Technology Center under grant agreement CCF-0939370, for the support during this research.

\small
\bibliographystyle{unsrt}
\bibliography{IEEEabrv,biblio}

\section{Appendix}

\subsection{Proof of Theorem \ref{MaxEntThm}}
\subsubsection{Weak Version}
First, we list the assumptions of the weak version of Theorem 1:
\begin{itemize}[leftmargin=.3in]
\item $\Gamma$ is convex and closed,
\item Loss function $L$ is bounded by a constant $C$,
\item $\mathcal{X},\mathcal{Y}$ are finite,
\item Risk set $S=\lbrace \, \left[ L(y,a)\right]_{y\in \mathcal{Y}} \, :\, a\in \mathcal{A}\,\rbrace $ is closed.
\end{itemize}
Given these assumptions, Sion's minimax theorem \cite{sion} implies that the minimax problem has a finite answer $H^*$,
\begin{equation} \label{Eq: Sion_minimax}
H^* := \sup_{P\in \Gamma}\, \inf_{\psi \in \boldsymbol{\Psi}}\: \mathbb{E}[L(Y,\psi(X))]= \inf_{\psi \in \boldsymbol{\Psi}}\, \sup_{P\in \Gamma}\: \mathbb{E}[L(Y,\psi(X))].
\end{equation}
Thus, there exists a sequence of decision rules $(\psi_n)_{n=1}^{\infty}$ for which
\begin{equation}
\lim_{n\rightarrow \infty} \, \sup_{P\in \Gamma}\, \mathbb{E}[L(Y,\psi_{n}(X))] \: =\: H^*.
\end{equation}
As we supposed, the risk set $S$ is closed. Therefore, the randomized risk set\footnote{$L(y,\zeta)$ is a short-form for $E[L(y,A)]$ where $A\in \mathcal{A}$ is a random action distributed according to $\zeta$.} $S_r=\lbrace \, \left[ L(y,\zeta)\right]_{y\in \mathcal{Y}} \, :\, \zeta\in \mathcal{Z}\,\rbrace $ defined over the space of randomized acts $\mathcal{Z}$ is also closed and, since $L$ is bounded, is a compact subset of $\mathbb{R}^{|\mathcal{Y}|}$. Therefore, since $\mathcal{X}$ and $\mathcal{Y}$ are both finite, we can find a randomized decision rule $\psi^*$ which on taking a subsequence $(n_k)_{k=1}^{\infty}$ 
satisfies
\begin{equation}
 \forall\: x\in\mathcal{X},\,y\in\mathcal{Y}:\quad  L(y,\psi^*(x)) \, =\, \lim_{k\rightarrow \infty}\, L(y,\psi_{n_k}(x)).
\end{equation}
Then $\psi^*$ is a robust Bayes decision rule against $\Gamma$, because
\begin{equation} \label{Eq: proof_weak_1}
\sup_{P\in \Gamma}\: \mathbb{E}\,[L(Y,\psi^*(X))] = \sup_{P\in \Gamma}\: \lim_{k\rightarrow\infty}\, \mathbb{E}\,[L(Y,\psi_{n_k}(X))] \le \lim_{k\rightarrow\infty}\, \sup_{P\in \Gamma}\,\mathbb{E}[L(Y,\psi_{n_k}(X))] = H^*.
\end{equation}
Moreover, since $\Gamma$ is assumed to be convex and closed (hence compact), $H(Y|X)$ achieves its supremum over $\Gamma$ at some distribution $P^*$. By the definition of conditional entropy, \eqref{Eq: proof_weak_1} implies that
\begin{equation}
E_{P^*}[L(Y,\psi^*(X))] \le \sup_{P\in \Gamma}\: \mathbb{E}\,[L(Y,\psi^*(X))] \le H^* = H_{P^*}(Y|X),
\end{equation}
which shows that $\psi^*$ is a Bayes decision rule for $P^*$ as well. This completes the proof.

\subsubsection{Strong Version}
Let's recall the assumptions of the strong version of Theorem 1:
\begin{itemize}[leftmargin=.3in]
\item $\Gamma$ is convex.
\item For any distribution $P\in \Gamma$, there exists a Bayes decision rule.
\item We assume continuity in Bayes decision rules over $\Gamma$, i.e., if a sequence of distributions $(Q_n)_{n=1}^{\infty} \in \Gamma$ with  the corresponding Bayes decision rules $(\psi_n)_{n=1}^{\infty}$ converges to $Q$ with a Bayes decision rule $\psi$, then under any $P\in \Gamma$, the expected loss of $\psi_n$ converges to the expected loss of $\psi$.
\item $P^*$ maximizes the conditional entropy $H(Y|X)$.
\end{itemize}
\textbf{Note:} A particular structure used in our paper is given by fixing the marginal $P_X$ across $\Gamma$. Under this structure, the condition of the continuity in Bayes decision rules reduces to the continuity in Bayes acts over $P_Y$'s in $\Gamma_{Y|X}$. It can be seen that while this condition holds for the logarithmic and quadratic loss functions, it does not hold for the 0-1 loss.

Let  $\psi^*$ be a Bayes decision rule for $P^*$. We need to show that  $\psi^*$ is a robust Bayes decision rule against $\Gamma$. To show this, it suffices to show that $(P^*,\psi^*)$ is a saddle point of the mentioned minimax problem, i.e., 
\begin{equation}
\label{eq:Firstone}
\mathbb{E}_{P^*}[L(Y,\psi^*(X))] \le \mathbb{E}_{P^*}[L(Y,\psi(X))],
\end{equation}
and 
\begin{equation} \label{Thm1: want_to_show}
\mathbb{E}_{P^*}[L(Y,\psi^*(X))] \ge \mathbb{E}_{P}[L(Y,\psi^*(X))].
\end{equation} 
Clearly,  inequality \eqref{eq:Firstone} holds due to the definition of the Bayes decision rule. To show \eqref{Thm1: want_to_show}, let us fix an arbitrary distribution $P\in \Gamma$. For any $\lambda\in(0,1]$, define $P_{\lambda}=\lambda P +(1-\lambda)P^*$. Notice that $P_{\lambda}\in \Gamma$ since $\Gamma$ is convex. Let $\psi_{\lambda}$ be a Bayes decision rule for $P_{\lambda}$. Due to the linearity of the expected loss in the probability distribution, we have
\begin{align*}
\mathbb{E}_{P}[L(Y,\psi_{\lambda}(X))] - \mathbb{E}_{P^*}[L(Y,\psi_{\lambda}(X))] &= \frac{\mathbb{E}_{P_{\lambda}}[L(Y,\psi_{\lambda}(X))] - \mathbb{E}_{P^*}[L(Y,\psi_{\lambda}(X))]}{\lambda} \\
 & \le \frac{H_{P_{\lambda}}(Y|X) - H_{P^*}(Y|X)}{\lambda} \\
 & \le 0,
\end{align*}
for any $0<\lambda\le 1 $. Here the first inequality is due to the definition of the conditional entropy and the last inequality holds since $P^*$  maximizes the conditional entropy over $\Gamma$. Applying the assumption of the continuity in Bayes decision rules, we have
\begin{equation}
\mathbb{E}_{P}[L(Y,\psi^*(X))]  - \mathbb{E}_{P^*}[L(Y,\psi^*(X))] =\lim_{\lambda \rightarrow 0}\, \mathbb{E}_{P}[L(Y,\psi_{\lambda}(X))] - \mathbb{E}_{P^*}[L(Y,\psi_{\lambda}(X))] \le 0,
\end{equation}
which makes the proof complete.

\subsection{Proof of Theorem \ref{Thm: duality} }
Let us recall the definition of the set $\Gamma(Q)$:
\begin{align}
\Gamma(Q)=\lbrace \: P_{\mathbf{X},Y}:& \: P_\mathbf{X}=Q_\mathbf{X} \, , \label{Eq: Classform_2} \\
&\: \forall\, 1\le i\le t:\;\;  \Vert\, \mathbb{E}_{P}\left[{\theta}_i(Y)\mathbf{X}\right] - \mathbb{E}_{Q}\left[ {\theta}_i(Y)\mathbf{X}\right]\Vert \le \epsilon_i \:\rbrace . \nonumber
\end{align}
Defining $\tilde{\mathbf{E}}_i \triangleq \mathbb{E}_{Q}\left[ {\theta}_i(Y)\mathbf{X}\right]$ and $C_i\triangleq\lbrace \mathbf{u}: \Vert \mathbf{u} - \tilde{\mathbf{E}}_i\Vert \le \epsilon_i\rbrace$, we have
\begin{equation} \label{Thm2: Main problem}
\max_{ P \in \Gamma(Q)} H(Y\vert \mathbf{X}) =
 \max_{P,\mathbf{w}:\; 
 \forall i:\: \mathbf{w}_i=\mathbb{E}_{P}\left[{\theta}_i(Y)\mathbf{X}\right]} \; \mathbb{E}_{Q_{\mathbf{X}}}\left[ H_P(Y|\mathbf{X}=\mathbf{x})\right] + \sum_{i=1}^{t}I_{C_i}(\mathbf{w}_i)
\end{equation}
where $I_C$ is the indicator function for the set $C$ defined as
\begin{equation}
I_C(x)=\begin{cases}
0 \quad & \text{\rm if}\: x\in C, \\
-\infty \quad & \text{\rm Otherwise.}
\end{cases}
\end{equation}
First of all, the law of iterated expectations implies that $\mathbb{E}_{P}\left[{\theta}_i(Y)\mathbf{X}\right]=\mathbb{E}_{Q_\mathbf{X}} \bigg[\,\mathbf{X}\, \mathbb{E}[{\theta}_i(Y)|\mathbf{X}=\mathbf{x}]\,\bigg]$.  Furthermore, \eqref{Thm2: Main problem} is equivalent to a convex optimization problem where it is not hard to check that the Slater condition is satisfied. Hence strong duality holds and we can write the dual problem as
\small
\begin{equation}
 \min_{\mathbf{A}}\sup_{P_{Y|\mathbf{X}},\mathbf{w}} \; \mathbb{E}_{Q_{\mathbf{X}}}\left[ H_P(Y|\mathbf{X}=\mathbf{x}) +  \sum_{i=1}^t\mathbb{E}[{\theta}_i(Y)|\mathbf{X}=\mathbf{x}]\mathbf{A}_i\mathbf{X}\right] + \sum_{i=1}^{t}\left[I_{C_i}(\mathbf{w}_i) -\mathbf{A}_i\mathbf{w}_i\right],
\end{equation}
\normalsize
where the rows of matrix $\mathbf{A}$, denoted by $\mathbf{A}_i$, are the Lagrange multipliers for the constraints of $\mathbf{w}_i=\mathbb{E}_{P}\left[{\theta}_i(Y)\mathbf{X}\right]$. Notice that the above problem decomposes across $P_{Y|\mathbf{X}=\mathbf{x}}$'s and $\mathbf{w}_i$'s. Hence, the dual problem can be rewritten as  
\small
\begin{align}
 \min_{\boldsymbol{A}}  \left[\mathbb{E}_{Q_{\mathbf{X}}}\left[\sup_{P_{Y|\mathbf{X} = \mathbf{x}}}   H_P(Y|\mathbf{X}=\mathbf{x}) +  \sum_{i=1}^t\mathbb{E}[{\theta}_i(Y)|\mathbf{X}=\mathbf{x}]\mathbf{A}_i\mathbf{X} \right] + \sum_{i=1}^{t} \sup_{\mathbf{w}_i}\left[I_{C_i}(\mathbf{w}_i) -\mathbf{A}_i\mathbf{w}_i\right] \right]\label{eq:temp1}
\end{align}
\normalsize
Furthermore, according to the definition of $F_{\boldsymbol{\theta}}$, we have
\begin{equation} \label{KKT}
F_{\boldsymbol{\theta}}(\mathbf{A}\mathbf{x}) = \sup_{P_{Y|\mathbf{X=x}}} \;  H(Y|\mathbf{X}=\mathbf{x}) +  \mathbb{E}[\boldsymbol{\theta}(Y)|\mathbf{X}=\mathbf{x}]^T\mathbf{A}\mathbf{x} .
\end{equation}
Moreover, the definition of the dual norm $\|\cdot\|_*$ implies
\begin{equation}
\label{temp2}
\sup_{\mathbf{w}_i} \; I_{C_i}(\mathbf{w}_i) -\mathbf{A}_i\mathbf{w}_i = \max_{\mathbf{u}\in C_i} -\mathbf{A}_i\mathbf{u} = -\mathbf{A}_i\tilde{\mathbf{E}}_i + \epsilon_i \Vert \mathbf{A}_i\Vert_*.
\end{equation}
Plugging \eqref{KKT} and \eqref{temp2} in \eqref{eq:temp1}, the dual problem can be simplified to
\begin{align}
&\min_{\mathbf{A}}\;\: \mathbb{E}_{Q_\mathbf{X}} \left[\, F_{\boldsymbol{\theta}} (\mathbf{A}\mathbf{X}) - \sum_{i=1}^t\mathbf{A}_i\tilde{\mathbf{E}}_i \, \right] + \sum_{i=1}^t {\epsilon_i} \Vert \mathbf{A}_i \Vert_{*}\nonumber \\
 = \;\;&\min_{\mathbf{A}}\;\: \mathbb{E}_{Q} \left[\, F_{\boldsymbol{\theta}} (\mathbf{A}\mathbf{X}) - \boldsymbol{\theta}(Y)^T\mathbf{A}\mathbf{X}\, \right] + \sum_{i=1}^t {\epsilon_i} \Vert \mathbf{A}_i \Vert_{*},
\end{align}
which is equal to the primal problem \eqref{Thm2: Main problem} since the strong duality holds. Furthermore, note that we can rewrite the definition given for $F_{\boldsymbol{\theta}}$ as
\begin{equation} \label{F_theta: convexconjugate}
F_{\boldsymbol{\theta}}(\mathbf{z})\, = \, \max_{\mathbf{E} \in \mathbb{R}^t} \; G(\mathbf{E})\, +\, \mathbf{E}^T\mathbf{z}, 
\end{equation}
where we define
\begin{equation}
G(\mathbf{E}) = \begin{cases}
\begin{split}
\max_{P\in \mathcal{P_Y}:\: \mathbb{E}[\boldsymbol{\theta}(Y)]=\mathbf{E}}H(Y) \quad &\text{\rm if}\; \lbrace P\in \mathcal{P_Y}: \mathbb{E}[\boldsymbol{\theta}(Y)]=\mathbf{E}\rbrace \neq \emptyset \\
-\infty \qquad \qquad & \text{\rm Otherwise.}
\end{split}
\end{cases}
\end{equation}
Observe that $F_{\boldsymbol{\theta}}$ is the convex conjugate of the convex $-G$. Therefore, applying the derivative property of convex conjugates \cite{rockafellar} to \eqref{KKT}, 
\begin{equation} \label{F: defnition_convexconjugate}
\mathbb{E}_{P^*} [\,\boldsymbol{\theta}(Y) \,\vert\, \mathbf{X}=\mathbf{x}\,] \in \partial  F_{\boldsymbol{\theta}} \, (\mathbf{A}^*\mathbf{x}).
\end{equation}
Here, $\partial  F_{\boldsymbol{\theta}}$ denotes the subgradient of $ F_{\boldsymbol{\theta}}$. Assuming $F_{\boldsymbol{\theta}}$ is differentiable at $\mathbf{A}^*\mathbf{x}$, \eqref{F: defnition_convexconjugate} implies that
\begin{equation}
\mathbb{E}_{P^*} [\,\boldsymbol{\theta}(Y) \,\vert\, \mathbf{X}=\mathbf{x}\,] = \nabla  F_{\boldsymbol{\theta}} \, (\mathbf{A}^*\mathbf{x}).
\end{equation}

\subsection{Proof of Theorem \ref{Thm: Learnability 1}}
First, we aim to show that
\begin{equation}
\max_{P\in \Gamma(\tilde{P})}\E[ L(Y,\hat{\psi}_n(\mathbf{X}))] \le  
\mathbb{E}_{\tilde{P}} \left[\, F_{\boldsymbol{\theta}} (\hat{\mathbf{A}}\mathbf{X}) - \boldsymbol{\theta}(Y)^T\hat{\mathbf{A}}\mathbf{X}\, \right] 
+ \sum_{i=1}^t {\epsilon_i} \Vert \hat{\mathbf{A}}_{n_i} \Vert_{*}
\end{equation}
where $\hat{\mathbf{A}}$ denotes the solution to the RHS of the duality equation in Theorem 2 for the empirical distribution $\hat{P}_n$. Similar to the duality proven in Theorem 2, we can show that
\begin{align*}
\max_{P\in \Gamma(\tilde{P})}\E[ L(Y,\hat{\psi}_n(\mathbf{X}))] &= \min_{\mathbf{A}}\,  \E_{\tilde{P}_X} \biggl[ \sup_{P_{Y|\mathbf{X}} \in \mathcal{P_Y}}\, \E\bigl[ L(Y,\hat{\psi}_n(\mathbf{X})) | \mathbf{X}=\mathbf{x} \bigr]  + \E[\boldsymbol{\theta}(Y) | \mathbf{X}=\mathbf{x} ]^T \mathbf{A} \mathbf{X} \biggr] \\
&\;\;\; - \E_{\tilde{P}}[\boldsymbol{\theta}(Y)^T \mathbf{A} \mathbf{X} ]  + \sum_{i=1}^t {\epsilon_i} \Vert {\mathbf{A}}_i \Vert_{*} \\
& \le \;  \E_{\tilde{P}_X} \biggl[\sup_{P_{Y|\mathbf{X=x}} \in \mathcal{P_Y}}\, \E\bigl[ L(Y,\hat{\psi}_n(\mathbf{X})) | \mathbf{X}=\mathbf{x} \bigr] 
+ \E[\boldsymbol{\theta}(Y) | \mathbf{X} ]^T \hat{\mathbf{A}} \mathbf{X} \biggr] \\
&\;\;\; - \E_{\tilde{P}}[\boldsymbol{\theta}(Y)^T \hat{\mathbf{A}}\mathbf{X} ] + + \sum_{i=1}^t {\epsilon_i} \Vert \hat{\mathbf{A}}_{i} \Vert_{*} \\
& = \mathbb{E}_{\tilde{P}} \left[\, F_{\boldsymbol{\theta}} (\hat{\mathbf{A}}\mathbf{X}) - \boldsymbol{\theta}(Y)^T\hat{\mathbf{A}}\mathbf{X}\, \right] 
+ \sum_{i=1}^t {\epsilon_i} \Vert \hat{\mathbf{A}}_{i} \Vert_{*}.
\end{align*}
Here we first upper bound the minimum by taking the specific $\mathbf{A}=\hat{\mathbf{A}}$. Then the equality holds because $\hat{\psi}_n$ is a robust Bayes decision rule against $\Gamma (\hat{P}_n)$ and therefore adding the second term based on $\hat{\mathbf{A}}\mathbf{x}$, $\hat{\psi}_n(\mathbf{x})$ results in a saddle point for the following problem
\begin{align*}
F_{\boldsymbol{\theta}} (\hat{\mathbf{A}}\mathbf{x}) &= \sup_{P\in\mathcal{P_Y}}\; {H(Y) + \E[\boldsymbol{\theta}(Y)]^T\hat{\mathbf{A}}\mathbf{x}} \\
& =  \sup_{P\in\mathcal{P_Y}}\; \inf_{\zeta\in \mathcal{Z}}\; { \E[L(Y,\zeta)]+ \E[\boldsymbol{\theta}(Y)]^T\hat{\mathbf{A}}\mathbf{x}} \\
& = \sup_{P\in\mathcal{P_Y}}\; { \E[L(Y,\hat{\psi}_n(\mathbf{x}))]+ \E[\boldsymbol{\theta}(Y)]^T\hat{\mathbf{A}}\mathbf{x}}. 
\end{align*}
Therefore, by Theorem 2 we have
\begin{align} \label{Eq: proof_Thm3_1}
& \max_{P\in \Gamma(\tilde{P})}\E[ L(Y,\hat{\psi}_n(\mathbf{X}))] \, -\, \max_{P\in \Gamma(\tilde{P})} \E [ L(Y,\tilde{\psi}(\mathbf{X}))] \: \le \\
&\: \mathbb{E}_{\tilde{P}} \bigl[ F_{\boldsymbol{\theta}} (\hat{\mathbf{A}}\mathbf{X}) - \boldsymbol{\theta}(Y)^T\hat{\mathbf{A}}\mathbf{X} \bigr] 
+ \sum_{i=1}^t {\epsilon_i} \Vert \hat{\mathbf{A}}_i \Vert_{*} - \mathbb{E}_{\tilde{P}} \bigl[ F_{\boldsymbol{\theta}} (\tilde{\mathbf{A}}\mathbf{X}) - \boldsymbol{\theta}(Y)^T\tilde{\mathbf{A}}\mathbf{X} \bigr] 
- \sum_{i=1}^t {\epsilon_i} \Vert \tilde{\mathbf{A}}_{i} \Vert_{*}. \nonumber
\end{align}
As a result, we only need to bound the uniform convergence rate in the other side of the duality. Note that by the definition of $F_{\boldsymbol{\theta}}$,
\begin{equation}
\forall\: P\in \mathcal{P_Y},\, \mathbf{z}\in \mathbb{R}^t:\quad F_{\boldsymbol{\theta}}(\mathbf{z}) - \E_P[\boldsymbol{\theta}(Y)]^T\mathbf{z} \ge H_P(Y) \ge 0.
\end{equation}
Hence, $\forall \, \mathbf{A}:\:F_{\boldsymbol{\theta}}(\mathbf{A}\mathbf{X}) - \E[\boldsymbol{\theta}(Y)]^T\mathbf{A}\mathbf{X} \ge 0$ and comparing the optimal solution to the RHS of the duality equation in Theorem 2 to the case $\mathbf{A}=\mathbf{0}$ implies that for any possible solution $\mathbf{A}^*$ 
\begin{equation}
\forall\: 1\le i\le t:\quad  \epsilon_i \Vert \mathbf{A}^*_i\Vert_q \le \sum_{j=1}^t \epsilon_j \Vert \mathbf{A}^*_j\Vert_q \le F_{\boldsymbol{\theta}}(\mathbf{0})= \max_{P\in\mathcal{P_Y}}\, H(Y) = M.
\end{equation} 
Hence, we only need to bound the uniform convergence rate in a bounded space where $\forall\: 1\le i\le t:  \Vert \mathbf{A}_i\Vert_q   \le \frac{M}{\epsilon_i}$. 
Also, applying the derivative property of the conjugate relationship indicates that $\partial F_{\boldsymbol{\theta}}(\mathbf{z})$ is a subset of the convex hull of $\lbrace \mathbb{E}[\boldsymbol{\theta}(Y)]:\, P\in \mathcal{P_Y}\rbrace$. Therefore, for any $u\in \partial F_{\boldsymbol{\theta}}(\mathbf{z})$ we have $|| u||_2 \le L$, and $F_{\boldsymbol{\theta}} (\mathbf{z}) - \boldsymbol{\theta}(Y){\mathbf{z}}$ is $2L$-Lipschitz in $\mathbf{z}$. As a result, since $||\mathbf{X}||_p\le B$ and $||\boldsymbol{\theta}(Y)||_2\le L$ for any $\mathbf{A},\mathbf{A}^{'}$ such that $\Vert \mathbf{A}_i \Vert_2 \le \frac{M}{\epsilon_i}$,
\begin{equation}
\forall \, \mathbf{x}, \mathbf{x}^{'}, y,y^{'}:\;\: [\, F_{\boldsymbol{\theta}} (\mathbf{A} \mathbf{x}) - \boldsymbol{\theta}(y)^T{\mathbf{A} \mathbf{x}}\, ] - [\, F_{\boldsymbol{\theta}} (\mathbf{A}^{'} \mathbf{x}^{'}) - \boldsymbol{\theta}(y^{'})^T{\mathbf{A}^{'} \mathbf{x}^{'}}\, ] \, \le \, \sum_{i=1}^{t}\frac{4BML}{\epsilon_i}.
\end{equation}

Consequently, we can apply standard uniform convergence results given convexity-Lipschitzness-boundedness \cite{bartlett,kakade} as well as the vector contraction inequality from \cite{Rademacher_multiclass} to show that for any $\delta >0$ with a probability at least $1-\delta$
\begin{align}
\forall \,\mathbf{A}\in \mathbb{R}^{d\times t}, & \, \Vert\mathbf{A}_i \Vert_2 \le \frac{M}{\epsilon_i}:\;\; \mathbb{E}_{\tilde{P}} \bigl[ F_{\boldsymbol{\theta}} (\mathbf{A}\mathbf{X}) - \boldsymbol{\theta}(Y)^T\mathbf{A}\mathbf{X} \bigr]  \\ -\,&\mathbb{E}_{\hat{P}_n} \bigl[ F_{\boldsymbol{\theta}} (\mathbf{A}\mathbf{X}) - \boldsymbol{\theta}(Y)^T\mathbf{A}\mathbf{X} \bigr]
\, \le \, \frac{4BLM}{\sqrt{n}}\sum_{i=1}^{t}\frac{1}{\epsilon_i}\,\biggl( \sqrt{2}\sqrt{p-1}+ \sqrt{\frac{\log(2/\delta)}{2}} \biggr). \nonumber 
\end{align}  
Therefore, considering $\hat{\mathbf{A}}$ and $\tilde{\mathbf{A}}$ as the solution to the dual problems corresponding to the empirical and underlying cases, for any $\delta >0$ with a probability at least $1-\delta/2$ 
\begin{align}  \label{Eq: proof_Thm3_2}
& \mathbb{E}_{\tilde{P}} \bigl[ F_{\boldsymbol{\theta}} (\hat{\mathbf{A}}\mathbf{X}) - \boldsymbol{\theta}(Y)^T\hat{\mathbf{A}}\mathbf{X} \bigr] 
+\sum_{i=1}^t\epsilon_i \Vert{\hat{\mathbf{A}}}_i \Vert_q \\
-&\mathbb{E}_{\hat{P}_n} \bigl[ F_{\boldsymbol{\theta}} (\hat{\mathbf{A}}\mathbf{X}) - \boldsymbol{\theta}(Y)^T\hat{\mathbf{A}}\mathbf{X} \bigr] 
- \sum_{i=1}^t\epsilon_i \Vert{\hat{\mathbf{A}}}_i \Vert_q \le 
\frac{4BLM}{\sqrt{n}}\sum_{i=1}^{t}\frac{1}{\epsilon_i}\,\biggl( \sqrt{2(p-1)}+ \sqrt{\frac{\log(4/\delta)}{2}} \biggr). \nonumber 
\end{align} 
Since $\hat{\mathbf{A}}$ is minimizing the objective for $Q=\hat{P}_n$, 
\begin{align} \label{Eq: proof_Thm3_3}
&\mathbb{E}_{\hat{P}_n} \bigl[ F_{\boldsymbol{\theta}} (\hat{\mathbf{A}}\mathbf{X}) - \boldsymbol{\theta}(Y)^T\hat{\mathbf{A}}\mathbf{X} \bigr] 
+ \sum_{i=1}^t\epsilon_i \Vert{\hat{\mathbf{A}}}_i \Vert_q  \\
- & \mathbb{E}_{\hat{P}_n} \bigl[ F_{\boldsymbol{\theta}} (\tilde{\mathbf{A}}\mathbf{X}) - \boldsymbol{\theta}(Y)^T\tilde{\mathbf{A}}\mathbf{X} \bigr] 
- \sum_{i=1}^t\epsilon_i \Vert{\tilde{\mathbf{A}}}_i \Vert_q  \, \le \, 0. \nonumber
\end{align}
Also, since $\tilde{\mathbf{A}}$ does not depend on the samples, the Hoeffding's inequality implies that with a probability at least $1-\delta/2$ 
\begin{align} \label{Eq: proof_Thm3_4}
& \mathbb{E}_{\hat{P}_n} \bigl[ F_{\boldsymbol{\theta}} (\tilde{\mathbf{A}}\mathbf{X}) - \boldsymbol{\theta}(Y)^T\tilde{\mathbf{A}}\mathbf{X} \bigr] 
+ \sum_{i=1}^t\epsilon_i \Vert{\tilde{\mathbf{A}}}_i \Vert_q \\
- & \mathbb{E}_{\tilde{P}} \bigl[ F_{\boldsymbol{\theta}} (\tilde{\mathbf{A}}\mathbf{X}) - \boldsymbol{\theta}(Y)^T\tilde{\mathbf{A}}\mathbf{X} \bigr] 
- \sum_{i=1}^t\epsilon_i \Vert{\tilde{\mathbf{A}}}_i \Vert_q  \, \le \, \sum_{i=1}^t\frac{2BML}{\epsilon_i}\sqrt{\frac{\log(4/\delta)}{2n}}. \nonumber
\end{align}
Applying the union bound, combining \eqref{Eq: proof_Thm3_2}, \eqref{Eq: proof_Thm3_3}, \eqref{Eq: proof_Thm3_4} shows that with a probability at least $1-\delta$, we have
\begin{align} \label{Eq: proof_Thm3_5}
& \mathbb{E}_{\tilde{P}} \bigl[ F_{\boldsymbol{\theta}} (\hat{\mathbf{A}}\mathbf{X}) - \boldsymbol{\theta}(Y)^T\hat{\mathbf{A}}\mathbf{X} \bigr] 
+\sum_{i=1}^t\epsilon_i \Vert{\hat{\mathbf{A}}}_i \Vert_q  \\
- & \mathbb{E}_{\tilde{P}} \bigl[ F_{\boldsymbol{\theta}} (\tilde{\mathbf{A}}\mathbf{X}) - \boldsymbol{\theta}(Y)^T\tilde{\mathbf{A}}\mathbf{X} \bigr] 
- \sum_{i=1}^t\epsilon_i \Vert{\tilde{\mathbf{A}}}_i \Vert_q  \, \le \, \sum_{i=1}^{t}\frac{4BLM}{\epsilon_i\sqrt{n}}\biggl( \sqrt{2(p-1)}+ \frac{3}{2}\sqrt{\frac{\log(4/\delta)}{2}} \biggr). \nonumber
\end{align} 
Given \eqref{Eq: proof_Thm3_1} and \eqref{Eq: proof_Thm3_5},  the proof is complete. Note that we can improve the result in the case $q=1$ by using the same proof and plugging in the Rademacher complexity of the $\ell_1$-bounded linear functions. Here we only replace $\sqrt{2(p-1)}$ in the above bound with $\sqrt{4\log (2d)}$. 


\subsection{0-1 Loss: maximum entropy machine}
\subsubsection{$F_{\boldsymbol{\theta}}$ derivation}
Given the defined one-hot encoding $\boldsymbol{\theta}$ we define $\tilde{\mathbf{z}}=(\mathbf{z},0)$ and represent each randomized decision rule $\zeta$ with its corresponding loss vector $\mathbf{L}\in \mathbb{R}^{t+1}$ such that $ L_i=L_{\text{\rm 0-1}}(i,\zeta)$ denotes the 0-1 loss suffered by $\zeta$ when $Y=i$. It can be seen that $\mathbf{L}$ is a feasible loss vector if and only if $\forall\, i:\: 0\le L_i\le 1$ and $\sum_{i=1}^{t+1} L_i =t$. Then,
\begin{equation}
F_{\boldsymbol{\theta}} (\mathbf{z} ) =  \underset{\substack{\mathbf{p}\in \mathbb{R}^{t+1}:\: \mathbf{1}^T\mathbf{p}=1, \\ \forall i:\, 0\le p_i}}{\max}\;\;\: \underset{\substack{\mathbf{L}\in \mathbb{R}^{t+1}:\: \mathbf{1}^T\mathbf{L}=t, \\ \forall i:\, 0\le L_i \le 1}}{\min} \sum_{i=1}^{t+1} p_i ( \tilde{z}_i + L_i). 
\end{equation}
Hence, Sion's minimax theorem implies that the above minimax problem has a saddle point. Thus,
\begin{equation}
F_{\boldsymbol{\theta}} (\mathbf{z} ) =  \; \underset{\substack{\mathbf{L}\in \mathbb{R}^{t+1}:\: \mathbf{1}^T\mathbf{L}=t, \\ \forall i:\, 0\le L_i \le 1}}{\min} \max_{1\le i\le t+1} \lbrace\tilde{z}_i + L_i \rbrace.
\end{equation}
Consider $\sigma$ as the permutation sorting $\tilde{\mathbf{z}}$ in a descending order and for simplicity let $\tilde{z}_{(i)}=\tilde{z}_{\sigma(i)}$. Then,
\begin{equation}
\forall 1\le k\le t+1: \quad \max_{1\le i\le t+1} \lbrace\tilde{z}_i + L_i \rbrace \ge \frac{1}{k}\sum_{i=1}^{k} [\tilde{z}_{\sigma(i)} + L_{\sigma(i)} ] \ge \frac{k-1+\sum_{i=1}^{k} \tilde{z}_{(i)}  }{k},
\end{equation}
which is independent of the value of $L_i$'s. Therefore,
\begin{equation} \label{Eq: Proof_mmSVM_1}
\max_{1\le k\le t+1} \frac{k-1+\sum_{i=1}^{k} \tilde{z}_{(i)}  }{k} \le F_{\boldsymbol{\theta}} (\mathbf{z} ) .
\end{equation}
On the other hand, if we let $k_{\max}$ be the largest index satisfying $\sum_{i=1}^{k_{\max}} [\tilde{z}_{(i)} - \tilde{z}_{(k_{\max})}] < 1$ and define 
\begin{equation}
\forall\, 1\le j\le t+1:\quad L^*_{ \sigma(j)} = \begin{cases}
\begin{split}
&\frac{k_{\max}-1+\sum_{i=1}^{k_{\max}} \tilde{z}_{(i)}  }{k_{\max}}  -  \tilde{z}_{(j)} \;\; \text{\rm if}\; \sigma(j)\le k_{\max} \\
&1\qquad\qquad\qquad\qquad\qquad\qquad \;\:\text{\rm if}\; \sigma(j)>k_{\max}, 
\end{split}
\end{cases}
\end{equation}
we can simply check that $\mathbf{L}^*$ is a feasible point since $\sum_{i=1}^{t+1}L^*_i=t$ and $L^*_{\sigma(k_{\max})}\le 1$ so for all $i$'s $L^*_{\sigma(i)}\le 1$. Also, $L^*_{\sigma(1)}\ge 0$ because $\tilde{z}_{(1)} - \tilde{z}_{(j)} < 1$ for any $j\le k_{\max}$, so for all $i$'s $L^*_{\sigma(i)}\ge 0$. Then for this $\mathbf{L}^*$ we have
\begin{equation}
F_{\boldsymbol{\theta}} (\mathbf{z} ) \le \max_{1\le i\le t+1} \lbrace\tilde{z}_i + L^*_i \rbrace = \frac{k_{\max}-1+\sum_{i=1}^{k_{\max}} \tilde{z}_{(i)}  }{k_{\max}}.
\end{equation}
Therefore, \eqref{Eq: Proof_mmSVM_1} holds with equality and achieves its maximum at $k=k_{\max}$,
\begin{equation}
F_{\boldsymbol{\theta}} (\mathbf{z} ) = \max_{1\le k\le t+1} \frac{k-1+\sum_{i=1}^{k} \tilde{z}_{(i)}  }{k} = \frac{k_{\max}-1+\sum_{i=1}^{k_{\max}} \tilde{z}_{(i)}  }{k_{\max}}.
\end{equation}
Moreover, $\mathbf{L}^*$ corresponds to a randomized robust Bayes act, where we select label $i$ according to the probability vector $\mathbf{p}^* = \mathbf{1}- \mathbf{L}^*$ that is
\begin{equation}
\forall\, 1\le j\le t+1:\quad p^*_{ \sigma(j)} = \begin{cases}
\begin{split}
&\frac{1 - \sum_{i=1}^{k_{\max}} \tilde{z}_{(i)}  }{k_{\max}}  +  \tilde{z}_{(j)} \;\; \text{\rm if}\; \sigma(j)\le k_{\max} \\
&0\qquad\qquad\qquad\qquad \quad \; \text{\rm if}\; \sigma(j)>k_{\max}.
\end{split}
\end{cases}
\end{equation}

Given $F_{\boldsymbol{\theta}}$ we can simply derive the gradient $\nabla F_{\boldsymbol{\theta}}$ to find the entropy maximizing distribution. Here if the inequality $\sum_{i=1}^{k_{\max}} [ \tilde{\mathbf{z}}_{\sigma(i)} - \tilde{\mathbf{z}}_{(k_{\max}+1)}\, ] \ge 1$ holds strictly, which is true almost everywhere on $\mathbb{R}^{t}$, 
\begin{equation} \label{maximum entropy machine: multilabel f_5}
\forall\: 1\le i\le t:\; \;\; \bigl({\nabla F_{\boldsymbol{\theta}}}(\mathbf{z})\bigr)_i = \begin{cases} \begin{split}
&1/k_{\max} \quad \text{\rm if}\: \sigma(i)\le k_{\max},  \\
& 0 \quad \text{\rm Otherwise.}
\end{split}
\end{cases}
\end{equation}
If the inequality does not strictly hold, $F_{\boldsymbol{\theta}}$ is not differentiable at $\mathbf{z}$; however, the above vector still lies in the subgradient $\partial F_{\boldsymbol{\theta}}(\mathbf{z})$.
\subsubsection{Sufficient Conditions for Applying Theorem 1.a}
As supposed in Theorem 1.a, the space $\mathcal{X}$ should be finite in order to apply that  result. Here, we show for the proposed structure on $\Gamma(Q)$ one can relax this condition while Theorem 1.a still holds. It is because, as shown in the proofs of Theorems 2 and 3, we have
\begin{align*}
\inf_{\psi \in \boldsymbol{\Psi}}\, \max_{P\in \Gamma(\tilde{P})}\E[ L(Y,\psi(\mathbf{X}))] &=\,  \inf_{\psi \in \boldsymbol{\Psi}}\, \min_{\mathbf{A}}\,  \E_{\tilde{P}_X} \biggl[ \sup_{P_{Y|\mathbf{X}} \in \mathcal{P_Y}}\, \E\bigl[ L(Y,\psi(\mathbf{X})) | \mathbf{X}=\mathbf{x} \bigr]   \\
&\;\;\; + \E[\boldsymbol{\theta}(Y) | \mathbf{X}=\mathbf{x} ]^T \mathbf{A} \mathbf{X} \biggr] - \E_{\tilde{P}}[\boldsymbol{\theta}(Y)^T \mathbf{A} \mathbf{X} ]  + \sum_{i=1}^t {\epsilon_i} \Vert {\mathbf{A}}_i \Vert_{*} \\
&=\,  \min_{\mathbf{A}}\,  \E_{\tilde{P}_X} \biggl[ \inf_{\psi(\mathbf{x}) \in \mathcal{Z}}\sup_{P_{Y|\mathbf{X}} \in \mathcal{P_Y}}\, \E\bigl[ L(Y,\psi(\mathbf{x})) | \mathbf{X}=\mathbf{x} \bigr]   \\
&\;\;\;+ \E[\boldsymbol{\theta}(Y) | \mathbf{X}=\mathbf{x} ]^T \mathbf{A} \mathbf{X} \biggr] - \E_{\tilde{P}}[\boldsymbol{\theta}(Y)^T \mathbf{A} \mathbf{X} ]  + \sum_{i=1}^t {\epsilon_i} \Vert {\mathbf{A}}_i \Vert_{*}.
\end{align*}
Therefore, given this structure the minimax problem decouples across different $\mathbf{x}$'s. Hence, the assumption of finite $\mathcal{X}$ is no longer needed, because as long as $\boldsymbol{\theta}$ is a bounded function (which is true for the one-hot encoding $\boldsymbol{\theta}$), the rest of assumptions suffice to guarantee the existence of a saddle point given $\mathbf{X}=\mathbf{x}$ for any $\mathbf{x}$. 

\subsection{Quadratic Loss: Linear Regression}
\subsubsection{$F_{\boldsymbol{\theta}}$ derivation}
Here, we find $F_{\boldsymbol{\theta}}(\mathbf{z})\, = \, \max_{P\in \mathcal{P_Y}}\,  H(Y) + \mathbb{E}[\boldsymbol{\theta}(Y)]^T\mathbf{z} $ for $\boldsymbol{\theta}(Y)=Y$ and $\mathcal{P_Y}=\lbrace P_Y:\, \E[Y^2]\le \rho^2\rbrace$. Since for quadratic loss $H(Y)=\text{\rm Var}(Y)= \E[Y^2] - \E[Y]^2$, the problem is equivalent to
\begin{equation}
F_{\boldsymbol{\theta}}(z)\, = \, \max_{\E[Y^2]\le\rho^2}\,  \E[Y^2] - \E[Y]^2 + z\mathbb{E}[Y]
\end{equation}
As $\E[Y]^2\le\E[Y^2]$, it can be seen for the solution $\E_{P^*}[Y^2]=\rho^2$ and therefore we equivalently solve
\begin{equation}
F_{\boldsymbol{\theta}}(z)\, = \, \max_{|\E[Y]|\le\rho}\,  \rho^2 - \E[Y]^2 + z\mathbb{E}[Y] = \begin{cases}
\rho^2 + z^2/4\quad &\text{\rm if}\; |z/2|\le \rho \\
\rho  |z| \quad & \text{\rm if}\; |z/2| > \rho.
\end{cases}
\end{equation}
\subsubsection{Applying Theorem 2 while restricting $\mathcal{P_Y}$}
For the quadratic loss, we first change $\mathcal{P_Y}=\lbrace P_Y:\, \E[Y^2]\le \rho^2\rbrace$ and then apply Theorem 2. Note that by modifying $F_{\theta}$ based on the new $\mathcal{P_Y}$ we also solve a modified version of the maximum conditional entropy problem
\begin{equation}
\max_{\substack{P:\: P_{\mathbf{X},Y} \in \Gamma(Q)\\ \forall \mathbf{x}:\: P_{Y|\mathbf{X}=\mathbf{x}}\in \mathcal{P_Y} }} H(Y\vert \mathbf{X}) 
\end{equation}
In the case $\mathcal{P_Y}=\lbrace P_Y:\, \E[Y^2]\le \rho^2\rbrace$ Theorem 2 remains valid given the above modification in the maximum conditional entropy problem. This is because the inequality constraint $\E[Y^2|\mathbf{X}=\mathbf{x}]\le \rho^2$ is linear in $P_{Y|\mathbf{X}=\mathbf{x}}$, and thus the problem is still convex and strong duality holds as well. Also, when we move the constraints of $\mathbf{w}_i=\mathbb{E}_{P}\left[{\theta}_i(Y)\mathbf{X}\right]$ to the objective function, we get a similar dual problem
\begin{equation}
 \min_{\mathbf{A}}\sup_{\substack{  P_{Y|\mathbf{X}},\mathbf{w}:\\ \forall \mathbf{x}:\: P_{Y|\mathbf{X}=\mathbf{x}}\in \mathcal{P_Y} }
 }  \mathbb{E}_{Q_{\mathbf{X}}}\left[ H_P(Y|\mathbf{X}=\mathbf{x}) +  \sum_{i=1}^t\mathbb{E}[{\theta}_i(Y)|\mathbf{X}=\mathbf{x}]\mathbf{A}_i\mathbf{X} \right] + \sum_{i=1}^{t}\left[I_{C_i}(\mathbf{w}_i) -\mathbf{A}_i\mathbf{w}_i\right]
\end{equation}
Following the next steps of the proof of Theorem 2, we complete the proof assuming the modification on $F_{\boldsymbol{\theta}}$ and the maximum conditional entropy problem.

\subsubsection{Derivation of group lasso}
To derive the group lasso problem, we slightly change the structure of $\Gamma (Q)$. 
First assume the subsets $I_1,\ldots,I_k$ are disjoint. Consider a set of distributions $\Gamma_{\text{\rm GL}}(Q)$ with the following structure
\begin{align}
\Gamma_{\text{\rm GL}}(Q)=\lbrace \: P_{\mathbf{X},Y}:& \: P_\mathbf{X}=Q_\mathbf{X} \, , \label{Eq: Classform_GroupLasso} \\
&\: \forall\, 1\le j\le k:\;\;  \Vert\, \mathbb{E}_{P}\left[Y\mathbf{X}_{I_j}\right] - \mathbb{E}_{Q}\left[ Y\mathbf{X}_{I_j}\right]\Vert \le \epsilon_j \:\rbrace . \nonumber
\end{align}
Now we prove a modified version of Theorem 2,
\begin{equation}  \label{GL: dual}
\max_{P\in\Gamma_{\text{\rm GL}}(Q)} H(Y\vert \mathbf{X}) \: =\min_{\boldsymbol{\alpha}}\: \mathbb{E}_{Q} \left[\, F_{\boldsymbol{\theta}} (\boldsymbol{\alpha}^T\mathbf{X}) - Y\boldsymbol{\alpha}^T\mathbf{X}\, \right] + \sum_{j=1}^k {\epsilon_j} \Vert \boldsymbol{\alpha}_{{I_j}} \Vert_{*}.
\end{equation}
To prove this identity, we can use the same proof provided for Theorem 2. We only need to redefine $\tilde{\mathbf{E}}_j=\mathbb{E}_{Q}\left[ Y\mathbf{X}_{I_j}\right]$ and $C_j=\lbrace \mathbf{u}: \Vert \mathbf{u} - \tilde{\mathbf{E}}_j\Vert \le \epsilon_j\rbrace$ for $1\le j\le k$. Notice that here $t=1$. Using the same technique in that proof, the dual problem can be formulated as
\begin{equation}
 \min_{\boldsymbol{\alpha}}\:\sup_{P_{Y|\mathbf{X}},\mathbf{w}} \; \mathbb{E}_{Q_{\mathbf{X}}}\left[ H_P(Y|\mathbf{X}=\mathbf{x}) +  \mathbb{E}[Y|\mathbf{X}=\mathbf{x}]\boldsymbol{\alpha}^T\mathbf{X} \right] + \sum_{j=1}^{k}\left[I_{C_j}(\mathbf{w}_{I_j}) -\boldsymbol{\alpha}_{I_j}\mathbf{w}_{I_j}\right].
\end{equation}
Similarly, we can decouple and simplify the above problem to derive the RHS of \eqref{GL: dual}. Then, if we let $\Vert\cdot\Vert$ be the $\ell_q$-norm, we will get the group lasso problem with the $\ell_{1,p}$ regularizer.

If the subsets are not disjoint, we can create new copies of each feature corresponding to a repeated index, such that there will be no repeated indices after adding the new features. Note that since $P_\mathbf{X}$ has been fixed over $\Gamma_{\text{\rm GL}}(Q)$ adding the extra copies of original features does not change the maximum-conditional entropy problem. Hence, we can use the result proven for the disjoint case and derive the overlapping group lasso problem.

\end{document}